\documentclass[journal]{IEEEtran}
\usepackage{amsmath,amsfonts, amssymb, amsthm}
\usepackage{algorithmic}
\usepackage{algorithm}
\usepackage{array}
\usepackage{textcomp}
\usepackage{stfloats}
\usepackage{url}
\usepackage{verbatim}
\usepackage{graphicx}

\ifCLASSOPTIONcompsoc
\usepackage[caption=false,font=normalsize,labelfont=sf,textfont=sf,subrefformat=parens,labelformat=parens]{subfig}
\else
\usepackage[caption=false,font=footnotesize,subrefformat=parens,labelformat=parens]{subfig}
\fi

\usepackage{cite}
\usepackage{color}
\usepackage{bm}
\newtheorem{lemma}{Lemma}
\newcommand{\response}[1]{{\color{black}#1}}

\begin{document}

\title{Model-Driven Learning-Based Physical Layer Authentication for Mobile Wi-Fi Devices}

\author{Yijia~Guo,~\IEEEmembership{Member,~IEEE},
Junqing~Zhang,~\IEEEmembership{Senior Member,~IEEE},\\
Y.-W.~Peter~Hong,~\IEEEmembership{Senior Member,~IEEE},
and Stefano~Tomasin,~\IEEEmembership{Senior Member,~IEEE}
\thanks{Manuscript received xxx; revised xxx; accepted xxx. Date of publication xxx; date of current version xxx. The work of J. Zhang was supported in part by the UK Engineering and Physical Sciences Research Council (EPSRC) under grant ID EP/Y037197/1 and in part by Royal Society Research Grants under grant ID RGS/R1/231435. The work of Y.-W.~P.~Hong was supported in part by the National Science and Technology Council (NSTC) of Taiwan under grant NSTC 114-2221-E-007-145-MY3. The work of S.~Tomasin has been funded by the European Commission through the Horizon Europe/JU SNS project ROBUST-6G (Grant Agreement no. 101139068). The work of J.~Zhang and S. Tomasin was also supported by the EU COST Action CA22168 - Physical layer security for trustworthy and resilient 6G systems (6G-PHYSEC). For the purpose of open access, the authors have applied a Creative Commons Attribution (CC BY) licence to any Accepted Manuscript version arising.
The review of this paper was coordinated by xxx. 
    \textit{(Corresponding author: Junqing Zhang.)}}
    \thanks{Y.~Guo is with the School of Computer Science and Informatics, University of Liverpool, Liverpool, L69 3DR, United Kingdom. She is also with the Institute of Communications Engineering, National Tsing Hua University, Hsinchu, Taiwan 300044. (email: yijia.guo@liverpool.ac.uk)}
	\thanks{J.~Zhang is with the School of Computer Science and Informatics, University of Liverpool, Liverpool, L69 3DR, United Kingdom. (email: junqing.zhang@liverpool.ac.uk)}
    \thanks{Y.-W. P. Hong is with the Institute of Communications Engineering, National Tsing Hua University, Hsinchu, Taiwan 300044. (email:ywhong@ee.nthu.edu.tw)}    
    \thanks{S.~Tomasin is with the Department of Information Engineering, University of Padova, Padova, Italy. (email: stefano.tomasin@unipd.it)}
	\thanks{Color versions of one or more of the figures in this paper are available online at http://ieeexplore.ieee.org.}
	\thanks{Digital Object Identifier xxx}	
}

\markboth{Journal of \LaTeX\ Class Files,~Vol.~14, No.~8, August~2021}%
{Shell \MakeLowercase{\textit{et al.}}: A Sample Article Using IEEEtran.cls for IEEE Journals}


\maketitle

\begin{abstract}
The rise of wireless technologies has made the Internet of Things (IoT) ubiquitous, but the broadcast nature of wireless communications exposes IoT to authentication risks. Physical layer authentication (PLA) offers a promising solution by leveraging unique characteristics of wireless channels. As a common approach in PLA, hypothesis testing yields a theoretically optimal Neyman-Pearson (NP) detector, but its reliance on channel statistics limits its practicality in real-world scenarios. In contrast, deep learning-based PLA approaches are practical but tend to be not optimal.
To address these challenges, we proposed a learning-based PLA scheme driven by hypothesis testing and conducted extensive simulations and experimental evaluations using Wi-Fi. Specifically, we incorporated conditional statistical models into the hypothesis testing framework to derive a theoretically optimal NP detector. Building on this, we developed LiteNP-Net, a lightweight neural network driven by the NP detector. Simulation results demonstrated that LiteNP-Net could approach the performance of the NP detector even without prior knowledge of the channel statistics. To further assess its effectiveness in practical environments, we deployed an experimental testbed using Wi-Fi IoT development kits in various real-world scenarios. Experimental results demonstrated that the LiteNP-Net outperformed the conventional correlation-based method as well as state-of-the-art Siamese-based methods.
\end{abstract}


\section{Introduction}
\label{sec:introduction}
\IEEEPARstart{W}{ith} the rapid advancement of wireless technologies, the Internet of Things (IoT) enables widespread connectivity and seamless communication~\cite{aouedi2025survey}. However, the inherent broadcast nature of wireless communication makes it vulnerable to security threats, as any device within the signal range can intercept transmissions. This exposes networks to spoofing attacks, where malicious entities can masquerade as legitimate users to gain unauthorized access.

To mitigate such threats, device authentication plays a crucial role in ensuring network security. Specifically, physical layer authentication (PLA) based on channel characteristics has attracted significant attention, owing to the inherent difficulty of replicating wireless channels~\cite{wang2020physical,xie2020survey,hoang2024physical,zhang2025physical}. In particular, most earlier works on PLA exploited received signal strength (RSS) as coarse channel characteristics~\cite{yang2009detecting, yang2013detection, chen2010detecting}, which limited detection accuracy. However, with the development of orthogonal frequency division multiplexing (OFDM), channel state information (CSI) has become a more robust and fine-grained alternative for authentication~\cite{liu2014practical, shi2017smart, liao2019novel}.

\response{Initial research on PLA primarily relied on statistical models, which directly compared the channel differences between estimated raw CSI measurements~\cite{xiao2007fingerprints, xiao2008mimo-assisted, xiao2008using, Tugnait2010channel-based}.}
To further enhance authentication accuracy,~\cite{liu2011robust} designed a test statistic based on the differences between noise-suppressed CSI measurements. Building on this,~\cite{liu2013two} incorporated multipath delay information with the channel difference to improve reliability under low signal-to-noise ratio (SNR) conditions. In~\cite{liu2016physical}, both the differences of the channel amplitude and the time intervals between adjacent paths were quantized and combined as the test statistic to further enhance performance. To mitigate performance degradation caused by quantization errors,~\cite{xie2021physical} proposed an authentication approach that leveraged the differences in multiple CSI measurements to strengthen network security. Additionally,~\cite{zhang2021exploiting} integrated CSI differences with transmitter-specific phase noise to further improve authentication robustness. Some variants of difference-based hypothesis testing, such as correlation-based methods, can be theoretically derived under the assumption of equal energy symbols. However, existing works, e.g.,~\cite{liu2018authenticating}, may use correlation-based detectors heuristically without properly computing the impact of symbol power, which may lead to suboptimal performance.
Although difference-based hypothesis testing has demonstrated some effectiveness in prior studies, it still remains a heuristic approach in determining the test criterion. In contrast,~\cite{xiao2009channel} introduced a principled approach by developing a conditional statistical model tailored to frequency-selective Rayleigh channels. Specifically, channel variations were modelled as conditional distributions given the observed CSI measurements and the decision rule was derived from the Neyman-Pearson (NP) criterion. However, this model overlooks noise effects, limiting its practical effectiveness. Furthermore, although hypothesis testing can, in theory, yield the optimal NP detector, its application is constrained by the need for channel statistics, which is difficult to obtain in practice and may vary in different scenarios and over time.

With the progress of data-driven technologies, various learning-based methods have emerged for mobile authentication, which can be grouped into three primary approaches. 1) The first group involves using channel prediction for authentication. In~\cite{germain2021channel}, long short-term memory (LSTM) and gated recurrent units (GRU) were utilized to predict future channel states based on past measurements, followed by mean square error (MSE)-based detection. Similarly,~\cite{wang2021channel} proposed legitimate CSI prediction using historical data and geographical information of the transmitter for authentication. However, these studies lack validation with real-world data, limiting their practical use.
2) The second group relies on the use of classification algorithms for authentication. In~\cite{pan2019threshold}, several classification algorithms such as decision trees (DT), support vector machines (SVM), K-nearest neighbours (KNN) and ensemble methods were investigated to classify CSI measurements in mobile scenarios. 
Additionally,~\cite{jing2024multi-user} applied a ResNet model to further refine the classification process. However, these methods are based on the assumption that legitimate and rogue devices are located in separate areas, resulting in non-overlapping CSI measurements that enable clear classification. Such an assumption may not always hold in practice, especially when mobile devices have overlapping trajectories.
3) The third group involves similarity-based authentication techniques. 
In our previous work~\cite{guo2025practical}, we presented a Siamese network for device authentication in mobile scenarios, where the embedding network was designed based on a convolutional neural network (CNN). In~\cite{zhang2025enhancing}, a novel method was proposed that incorporated a sliding window with a Siamese network, in which the embedding network was built upon a fully connected network (FCN). These methods leverage the Siamese network’s ability to learn and compare the similarity between feature representations, making them robust in identifying subtle differences between CSI measurements. \response{Overall, among the above approaches, classification-based methods are unsuitable as benchmarks due to their strong assumptions on spatial separability. Both channel prediction–based and similarity-based approaches rely on pairwise comparison between the newly received CSI and a reference CSI, where the reference CSI is obtained from prediction or from previously received legitimate CSI, respectively. As such comparisons are mainly conducted using correlation-based metrics or Siamese network architectures in existing works, we adopt the correlation-based method in~\cite{xiao2009channel} and the Siamese network–based methods in~\cite{guo2025practical, zhang2025enhancing} as benchmarks in this paper.}

Despite the effectiveness of the above methods, the vast majority of the learning-based methods are data-driven and treat the authentication process as a black box. As a result, the neural networks are not specifically optimized, leading to higher computational overhead and potentially suboptimal authentication performance. \response{Beyond data-driven approaches, model-driven learning-based methods have been proposed to incorporate communication domain knowledge for enhanced authentication performance. However, existing model-driven PLA studies mainly focus on theoretical convergence analysis rather than practical network design. It is demonstrated in~\cite{brighente2019machine} that both multilayer perceptron (MLP) and least square SVM converge to the NP detection when provided with an infinite training dataset and sufficiently complex learning models. Building on this, the comparison between the statistical-based and the learning-based methods was explored in~\cite{senigagliesi2021comparison}. Further, one-class classifiers were investigated in~\cite{ardizzon2024learning}, proving that one-class least square SVM converges to the generalized likelihood ratio test (GLRT) with a carefully designed transformation function. Its further application in the fifth-generation (5G) networks was conducted in~\cite{varotto2024one}. Despite their theoretical insights, existing studies mainly employ standard neural networks and lack a systematic methodology for network design based on the mathematical authentication model.}
Nevertheless, model-driven learning-based approaches have been successfully applied in other areas of physical layer communications, such as transmission schemes~\cite{oshea2017introduction}, receiver design~\cite{gao2018comnet} and CSI estimation~\cite{he2018deep}, offering valuable insights for its application in PLA.

In this paper, we designed a learning-based PLA scheme driven by hypothesis testing and conducted extensive simulation and experimental evaluations. Specifically, an optimal NP detector was derived under a hypothesis testing framework that explicitly accounts for CSI measurement noise. Furthermore, by utilizing the mathematical model of the NP detector, the model-driven approach facilitates the design of a lightweight neural network (LiteNP-Net) that converges to theoretically optimal performance without relying on channel statistics. Both simulation and experimental results demonstrate the reliability and robustness of the proposed scheme. Our key contributions are summarized as follows.
\begin{itemize}
    \item An NP detector was derived under a hypothesis testing framework. Specifically, the hypothesis testing framework was established using conditional statistical models of channel variation, which explicitly incorporated the effects of CSI measurement noise. Within this framework, the likelihood ratio was used to construct the NP detector.
    \item \response{Driven by the mathematical model of the NP detector, we proposed a systematic methodology that guides the design of a lightweight neural network, LiteNP-Net. Following this methodology, the architecture of LiteNP-Net is directly inspired by the structure of the NP detector, ensuring convergence to the optimal NP detector without relying on channel statistics.}
    \item We conducted an extensive simulation evaluation of the proposed scheme using various channel conditions in MATLAB. The impact of the LiteNP-Net architecture, transmission intervals, SNR and attack distance on authentication performance was examined. The simulation results demonstrated that the LiteNP-Net could approach the performance of the theoretically optimal detector without relying on channel statistics.
    \item We carried out a comprehensive experimental evaluation using Wi-Fi in different indoor environments. A testbed was constructed with an ESP32 kit and two LoPy4 boards. 
The evaluation considered both line-of-sight (LOS) and non-line-of-sight (NLOS) indoor conditions. 
The experimental dataset used in this paper can be found in~\cite{dataset_csi}.
The reliability of the LiteNP-Net was investigated, and the experimental results demonstrated that the LiteNP-Net outperformed conventional correlation-based method and state-of-the-art Siamese-based methods.
\end{itemize}

The rest of this paper is organized as follows. Section~\ref{sec:system} introduces the system model and problem statement. Section~\ref{sec:channel} explains the channel model, including temporal and spatial correlation characteristics. Section~\ref{sec:hypothesis} specifies the hypothesis testing analysis of the authentication problem. Section~\ref{sec:proposed} elaborates the LiteNP-Net driven by hypothesis testing. Section~\ref{sec:benchmark} describes the benchmarks and performance metrics used in this paper. Section~\ref{sec:simulation} and Section~\ref{sec:experimental} present and discuss the simulation and the experimental results, respectively. \response{Section~\ref{sec:discussion} provides the discussion about threshold selection and potential limitations.} Finally, Section~\ref{sec:conclusion} concludes the paper.

\section{System Model and Problem Statement}
\label{sec:system}
\subsection{System Model}
As depicted in Fig.~\ref{fig:SystemModel}, the legitimate user Bob sends packets to Alice over a dynamic and time-varying wireless channel. However, their communication is threatened by a potential attacker, Mallory, who attempts to impersonate Bob and send packets to Alice, thereby compromising the system security.

In this system, all users operate under the IEEE 802.11n legacy OFDM framework, which utilizes a channel bandwidth of $20$ MHz and a total of $M=64$ subcarriers. Long training symbol (LTS) is used for channel estimation, occupying $M^{\prime}=52$ active subcarriers. Let $X^{[k]}[m]$ denote the LTS modulated to the $m$-th subcarrier in the $k$-th packet, which is a binary phase shift keying (BPSK)-modulated symbol drawn from a pseudo-random sequence of $\pm 1$. The corresponding time-domain signal can be given as
\begin{equation}
x^{[k]}[n]=\frac{1}{\sqrt{M}}\sum_{m=0}^{M-1} X^{[k]}[m] e^{j 2 \pi m n / M}.
\end{equation}
\begin{figure}[!t]
\centerline{\includegraphics[width=3.4in]{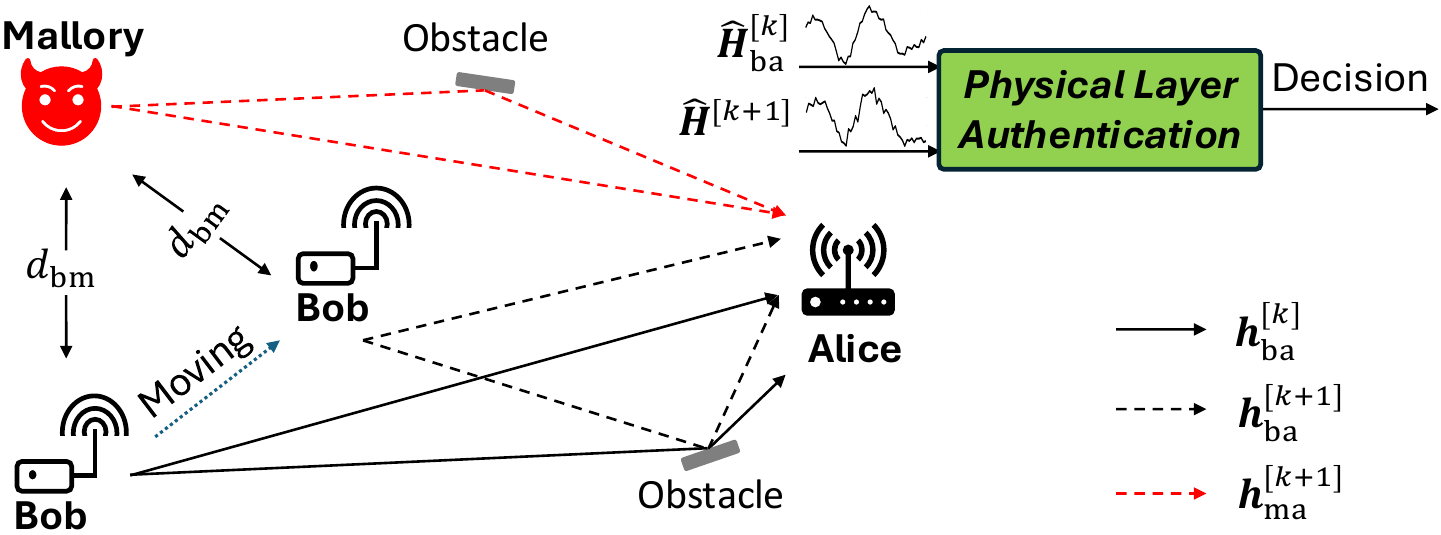}}
\caption{The system model for PLA in a mobile scenario. Multipath varies due to Bob's movement.}
\label{fig:SystemModel}
\end{figure}

Suppose the $k$-th packet received by Alice is sent by Bob at time $t_k$. The wireless channel between Bob and Alice at time $t_{k}$ can be expressed as $\bm{h}_{\rm ba}^{[k]}\triangleq[h_{\rm ba}^{[k]}[0],\dots,h_{\rm ba}^{[k]}[l],\dots,h_{\rm ba}^{[k]}[L-1]]$, where $L$ denotes the total number of channel taps and $h_{\rm ba}^{[k]}[l] \sim \mathcal{C N}(0, \sigma_{\rm ba}^2(l))$. Therefore, the received LTS in the $k$-th packet can be given as
\begin{equation}
y^{[k]}[n] =\sum_{l=0}^{L-1} x^{[k]}[n-l] h_{\rm ba}^{[k]}[l]+z[n],
\end{equation}
where $z[n] \sim \mathcal{C N}(0, \sigma_{0}^2)$ is the additive Gaussian white noise (AWGN). The equivalent frequency domain signal can be written as
\begin{equation}
\begin{aligned}
    Y^{[k]}[m] & =\frac{1}{\sqrt{M}}\sum_{n=0}^{M-1} y^{[k]}[n] e^{-j 2 \pi m n / M} \\
    & = X^{[k]}[m] H_{\rm ba}^{[k]}[m] +Z[m],
\end{aligned}
\end{equation}
where $H_{\rm ba}^{[k]}[m]$ is the channel frequency response (CFR) on the $m$-th subcarrier, given as
\begin{equation}
\label{eq:channeltimetofrequency}
    H_{\rm ba}^{[k]}[m]=\sum_{l=0}^{L-1} h_{\rm ba}^{[k]}[l] e^{-j 2 \pi m l/M},
\end{equation}
and
\begin{equation}
    Z[m]=\frac{1}{\sqrt{M}}\sum_{n=0}^{M-1} z[n] e^{-j 2 \pi m n / M}.
\end{equation}
The estimated channel coefficients across $M$ subcarriers can be expressed as $\widehat{\bm H}_{\rm ba}^{[k]}=[\widehat{H}_{\rm ba}^{[k]}[0], \dots, \widehat{H}_{\rm ba}^{[k]}[m], \dots, \widehat{H}_{\rm ba}^{[k]}[M-1]]$, which represents the CSI extracted from the received signal based on the least square (LS) channel estimation, with
\begin{equation}
\label{eq:lsEstimation_element}
\widehat{H}_{\rm ba}^{[k]}[m]=\frac{Y^{[k]}[m]}{X^{[k]}[m]}=H_{\rm ba}^{[k]}[m] +\widehat{Z}[m],
\end{equation}
where $\widehat{Z}[m] \sim \mathcal{CN}(0, \sigma^2)$ denotes the channel estimation noise, and we have $\sigma^2 = \sigma_{0}^2$ due to the unit power of $X^{[k]}[m]$.
Furthermore,~\eqref{eq:lsEstimation_element} can be rewritten in vector form as
\begin{equation}
\label{eq:lsEstimation_vector}
\bm{\widehat{H}}_{\rm ba}^{[k]} =\bm{H}_{\rm ba}^{[k]}+\bm{\widehat{Z}},
\end{equation}
where the true channel $\bm{H}_{\rm ba}^{[k]}\sim \mathcal{C N}(\bm{0}, \bm{\Sigma}_{\bm{H}})$ and channel estimation noise $\bm{\widehat{Z}} \sim \mathcal{C N}(\bm{0}, \sigma^2 \bm{I})$ are independent, with $\bm{I}$ denoting the identity matrix. Consequently, the estimated channel can be modelled as $\bm{\widehat{H}}_{\rm ba}^{[k]}\sim \mathcal{C N}(\bm{0}, \bm{\Sigma}_{\bm{\widehat{H}}})$, where 
\begin{equation}
\bm{\Sigma}_{\bm{\widehat{H}}}=\bm{\Sigma}_{\bm{H}}+\sigma^2 \bm{I}.
\end{equation}

\textbf{Threat Model:}
This paper focuses on a scenario where the attacker, Mallory, located at a distance $d_{\rm bm}$ from Bob, attempts to impersonate Bob, e.g., by spoofing his media access control (MAC) address, and sends packets to Alice.

The channel between Mallory and Alice at time $t_{k}$ is written as $\bm{h}_{\rm ma}^{[k]}\triangleq[h_{\rm ma}^{[k]}[0],\dots,h_{\rm ma}^{[k]}[l],\dots,h_{\rm ma}^{[k]}[L-1]]$, where $L$ refers to the number of channel taps. Without loss of generality, we assume the same tap length for $\bm{h}_{\rm ma}^{[k]}$ and $\bm{h}_{\rm ba}^{[k]}$.

Mallory, being aware of the wireless protocol utilized by Alice and Bob, has access to critical parameters such as the channel bandwidth and carrier frequency. Armed with this knowledge, Mallory is able to impersonate Bob and deceive Alice into accepting her as the legitimate transmitter by transmitting packets toward Alice.

\subsection{Problem Statement}
When Alice receives the $(k+1)$-th packet, originating either from Bob or Mallory, Alice estimates the CSI $\widehat{\bm H}^{[k+1]}$ from the received packet and performs device authentication to ensure the integrity of the communication. Conventionally, based on statistical analysis, this problem is formulated as a hypothesis testing problem, allowing the design of a NP detector to enable Alice to achieve theoretically optimal device authentication~\cite[Chapter 3]{kay1998detection}. However, implementing a NP detector requires sufficient knowledge of the channel statistics, which is often unrealistic in practice.

Beyond statistical methods, a substantial amount of research has explored data-driven approaches, in which the authentication task is treated as a black box, relying solely on neural networks without leveraging communication-specific domain knowledge. Despite the accuracy of the learning-based methods, they often incur considerable computational overhead. Furthermore, the performance gap between learning-based methods and the theoretical optimum remains largely unexplored, despite its significance.

This paper addressed the above challenges by using a model-driven neural network. It leveraged the advantages of both statistical and data-driven methods, enabling a lightweight network architecture while achieving performance close to the theoretical optimum.

\section{Channel Model}
\label{sec:channel}
In wireless communications, the time-varying nature of the channel can be represented by an autocorrelation function $R(\cdot)$. For the Bob-Alice link, the autocorrelation coefficient between $h_{\rm ba}^{[k+1]}[l]$ and $h_{\rm ba}^{[k]}[l]$ can be given as 
\begin{equation}
    \label{eq:alpha}
    \alpha \triangleq R(\Delta t_{k}),
\end{equation}
where $\Delta t_{k}=t_{k+1}-t_{k}$ denotes the transmission interval. For the Mallory-Alice link, the correlation between $h_{\rm ma}^{[k+1]}[l]$ and $h_{\rm ba}^{[k]}[l]$ can be expressed as the product of the spatial correlation $\rho(d_{\rm bm})$ between $h_{\rm ma}^{[k+1]}[l]$ and $h_{\rm ba}^{[k+1]}[l]$, and the temporal correlation $R(\Delta t_{k})$ between $h_{\rm ba}^{[k+1]}[l]$ and $h_{\rm ba}^{[k]}[l]$, i.e.,
\begin{equation}
    \label{eq:beta}
    \beta \triangleq \rho(d_{\rm bm}) \cdot R(\Delta t_{k}).
\end{equation}
Assuming Mallory moves at a velocity of $v_0$, $\rho(d_{\rm bm})$ can be equivalently regarded as the temporal correlation over the time period during which Mallory moves a distance of $d_{\rm bm}$~\cite[Chapter 3]{goldsmith2005wireless}, that is
\begin{equation}
    \label{eq:spatialcorrelation}
    \rho(d_{\rm bm})=R\Big(\frac{d_{\rm bm}}{v_0}\Big)=R\Big(\frac{d_{\rm bm}}{f_{\rm d}\lambda}\Big),
\end{equation}
where $f_{\rm d}$ represents the maximum Doppler spread and $\lambda$ is the wavelength. Therefore, by substituting~\eqref{eq:spatialcorrelation} into~\eqref{eq:beta}, the correlation between $h_{\rm ma}^{[k+1]}[l]$ and $h_{\rm ba}^{[k]}[l]$ is rewritten as
\begin{equation}
    \beta \triangleq R\Big(\frac{d_{\rm bm}}{f_{\rm d}\lambda}\Big) \cdot R(\Delta t_{k}).
\end{equation}

Assume that all the channel taps in $\bm{h}_{\text{tr}}^{[k]}$ adhere to the same Doppler spectrum, where $\text{tr} \in \{\text{ba}, \text{ma}\}$ denotes the link between the transmitter (Bob or Mallory) and the receiver (Alice), the channel at time $t_{k+1}$ can be expressed as~\cite{xiao2008using}
\begin{equation}\label{eq:channelcorrelation_time}
h_{\text{tr}}^{[k+1]}[l] =
\begin{cases}
\alpha h_{\rm ba}^{[k]}[l]+\sqrt{1-\alpha^2} \omega_{1}^{[k+1]}[l], & \text{if}\ \text{tr}=\text{ba}; \\
\frac{\beta}{\sqrt{\Theta}} h_{\rm ba}^{[k]}[l] + \frac{\sqrt{1-\beta^2}}{\sqrt{\Theta}} \omega_{2}^{[k+1]}[l], & \text{if}\ \text{tr}=\text{ma},
\end{cases}
\end{equation}
where the random noise component $\omega_{1}^{[k+1]}[l]$ and $\omega_{2}^{[k+1]}[l]$ are independent and identically distributed as $\mathcal{C N}(0, \sigma_{\rm ba}^2(l))$ and $\Theta = \sigma_{\rm ba}^2(l) / \sigma_{\rm ma}^2(l)$ is assumed constant across taps. The corresponding CFR can be expressed as 
\begin{equation}\label{eq:channelcorrelation_freq}
H_{\text{tr}}^{[k+1]}[m] =
\begin{cases}
\alpha H_{\rm ba}^{[k]}[m]+\sqrt{1-\alpha^2} \Omega_{1}^{[k+1]}[m], & \text{if}\ \text{tr}=\text{ba}; \\
\frac{\beta}{\sqrt{\Theta}} H_{\rm ba}^{[k]}[m]+\frac{\sqrt{1-\beta^2}}{\sqrt{\Theta}} \Omega_{2}^{[k+1]}[m], & \text{if}\ \text{tr}=\text{ma},
\end{cases}
\end{equation}
where $\Omega_{1}^{[k+1]}[m]$ and $\Omega_{2}^{[k+1]}[m]$ are the frequency domain representations of $\omega_{1}^{[k+1]}[l]$ and $\omega_{2}^{[k+1]}[l]$, respectively.

\section{Hypothesis Testing}
\label{sec:hypothesis}
In this section, we formulate device authentication as a hypothesis testing problem. Given that Alice receives the $k$-th packet from Bob, we aim to detect the presence of Mallory based on the CSI measured from the $(k+1)$-th packet. We first present the theoretically optimal NP detector based on conditional statistical models of channel variation and then declare the limitation of hypothesis testing.

\subsection{NP Detection}
\response{Given that Alice receives the $k$-th packet from Bob, a hypothesis test is used to decide if the $(k+1)$-th packet is transmitted by Bob or Mallory. The null hypothesis, $\mathcal{H}_0$, is that the $(k+1)$-th packet is transmitted by Bob. The alternative hypothesis, $\mathcal{H}_1$, is that the $(k+1)$-th packet is transmitted by Mallory. Thus,
\begin{subequations}
\begin{align}
    \mathcal{H}_0: \quad \bm{\widehat{H}}^{[k+1]} &= \bm{\widehat{H}}_{\rm ba}^{[k+1]} = \bm{H}_{\rm ba}^{[k+1]} + \bm{\widehat{Z}}, \\
    \mathcal{H}_1: \quad \bm{\widehat{H}}^{[k+1]} &= \bm{\widehat{H}}_{\rm ma}^{[k+1]} = \bm{H}_{\rm ma}^{[k+1]} + \bm{\widehat{Z}}.
\end{align}
\end{subequations}
Under null hypothesis $\mathcal{H}_0$, the conditional distribution of $\bm{\widehat{H}}_{\rm ba}^{[k+1]}$ given $\bm{\widehat{H}}_{\rm ba}^{[k]}$ is 
\begin{equation}
    \bm{\widehat{H}}_{\rm ba}^{[k+1]} \mid \bm{\widehat{H}}_{\rm ba}^{[k]} \sim \mathcal{CN}(\bm{\mu}_{\rm ba}^{[k+1]}, \bm{\Sigma}_{\rm ba}^{[k+1]}).
\end{equation}
Under alternative hypothesis $\mathcal{H}_1$, the conditional distribution of $\bm{\widehat{H}}_{\rm ma}^{[k+1]}$ given $\bm{\widehat{H}}_{\rm ba}^{[k]}$ is
\begin{equation}
    \bm{\widehat{H}}_{\rm ma}^{[k+1]} \mid \bm{\widehat{H}}_{\rm ba}^{[k]} \sim \mathcal{CN}(\bm{\mu}_{\rm ma}^{[k+1]}, \bm{\Sigma}_{\rm ma}^{[k+1]}).
\end{equation}
As demonstrated in Appendix~\ref{proof:pdf}, the conditional covariances $\bm{\Sigma}_{\rm ba}^{[k+1]}$ and $\bm{\Sigma}_{\rm ma}^{[k+1]}$ do not depend on the index $[k+1]$ and are thus written simply as $\bm{\Sigma}_{\rm ba}$ and $\bm{\Sigma}_{\rm ma}$, respectively. Furthermore, Appendix \ref{proof:pdf} also establishes the NP detector based on the likelihood ratio, which is given as
\begin{equation}
\label{eq:hypothesiscomplex}
\begin{aligned}
\Lambda \triangleq & (\bm{\widehat{H}}^{[k+1]})^{H} \bm{A} \bm{\widehat{H}}^{[k+1]}+\Re\{(\bm{\widehat{H}}_{\rm ba}^{[k]})^{H}\bm{B}\bm{\widehat{H}}^{[k+1]}\}\\
& +(\bm{\widehat{H}}_{\rm ba}^{[k]})^{H} \bm{C} \bm{\widehat{H}}_{\rm ba}^{[k]} \underset{\mathcal{H}_1}{\overset{\mathcal{H}_0}{\gtrless}}  \mathcal{T},
\end{aligned}
\end{equation}
where $\Re(\cdot)$ represents the real part and 
\begin{subequations}
\begin{align}
    & \bm{A}=-(\bm{\Sigma}_{\rm ba}^{-1}-\bm{\Sigma}_{\rm ma}^{-1}),\\
    & \bm{B}=2(\bm{\Sigma}_{\bm{H}} \bm{\Sigma}_{\bm{\widehat{H}}}^{-1})^{H}(\alpha \bm{\Sigma}_{\rm ba}^{-1}-\frac{\beta}{\sqrt{\Theta}}\bm{\Sigma}_{\rm ma}^{-1}),\\
    & \bm{C}=-\bm{\Sigma}_{\bm{\widehat{H}}}^{-1} \bm{\Sigma}_{\bm{H}} (\alpha^2 \bm{\Sigma}_{\rm ba}^{-1}-\frac{\beta^2}{\Theta} \bm{\Sigma}_{\rm ma}^{-1}) \bm{\Sigma}_{\bm{H}} \bm{\Sigma}_{\bm{\widehat{H}}}^{-1},
\end{align}
\end{subequations}
are related to the channel statistics. 
}

\subsection{Limitation of Hypothesis Testing}
As can be seen from~\eqref{eq:hypothesiscomplex}, the NP detector completely depends on channel statistics, i.e., $\alpha$, $\beta$, $\bm{\Sigma}_{\bm{H}}$, $\bm{\Sigma}_{\bm{\widehat{H}}}$, $\bm{\Sigma}_{\rm ba}$ and $\bm{\Sigma}_{\rm ma}$.
Assuming that the receiver obtains $N_{\rm ch}$ channel measurements and the channel estimation noise power $\sigma^2$, $\alpha$ and $\beta$ can be estimated as
\begin{subequations}
\begin{align}
\label{eq:alpha_eatimation}
& \hat{\alpha}=\Re\Big\{\frac{\sum_{i=1}^{N_{\rm ch}} \{\widehat{H}_{\rm ba}^{[k+1]}[m]\}_{i}\cdot \{\widehat{H}_{\rm ba}^{[k]*}[m]\}_{i}}{\sum_{i=1}^{N_{\rm ch}} \{\widehat{H}_{\rm ba}^{[k]}[m]\}_{i}\cdot \{\widehat{H}_{\rm ba}^{[k]*}[m]\}_{i}-\sigma^2} \Big\},\\
\label{eq:beta_eatimation}
& \hat{\beta}=\sqrt{\Theta}\cdot
\Re\Big\{ \frac{\sum_{i=1}^{N_{\rm ch}} \{\widehat{H}_{\rm ma}^{[k+1]}[m]\}_{i}\cdot \{\widehat{H}_{\rm ba}^{[k]*}[m]\}_{i}}{\sum_{i=1}^{N_{\rm ch}} \{\widehat{H}_{\rm ba}^{[k]}[m]\}_{i}\cdot \{\widehat{H}_{\rm ba}^{[k]*}[m]\}_{i}-\sigma^2} \Big\}.
\end{align}
\end{subequations}
The covariance matrices $\bm{\Sigma}_{\bm{\widehat{H}}}$ and $\bm{\Sigma}_{\bm{H}}$ can be estimated as
\begin{subequations}
\begin{align}
\label{eq:Sigma_hat_H_eatimation}
& \widehat{\bm \Sigma}_{\widehat{\bm H}}=\frac{1}{N_{\rm ch}}\sum_{i=1}^{N_{\rm ch}} \{\widehat{\bm H}_{\rm ba}^{[k]}\}_{i}\cdot \{\widehat{\bm H}_{\rm ba}^{[k]}\}_{i}^{H},\\
\label{eq:Sigma_H_eatimation}
& \widehat{\bm \Sigma}_{\bm{H}}=\frac{1}{N_{\rm ch}}\sum_{i=1}^{N_{\rm ch}} \{\widehat{\bm H}_{\rm ba}^{[k]}\}_{i}\cdot \{\widehat{\bm H}_{\rm ba}^{[k]}\}_{i}^{H}-\sigma^2 \bm{I}.
\end{align}
\end{subequations}
Similarly, the covariance matrices ${\bm \Sigma}_{\rm ba}$ and ${\bm \Sigma}_{\rm ma}$ can be estimated as
\begin{subequations}
\begin{align}
&
\label{eq:sigma_ba_estimation}
\begin{aligned}
\widehat{\bm \Sigma}_{\rm ba} =& \frac{1}{N_{\rm ch}}\sum_{i=1}^{N_{\rm ch}} [\{\widehat{\bm H}_{\rm ba}^{[k+1]}\}_{i}-\{{\bm \mu}_{\rm ba}^{[k+1]}\}_{i}]\cdot \\
& [\{\widehat{\bm H}_{\rm ba}^{[k+1]}\}_{i}^{H}-\{{\bm \mu}_{\rm ba}^{[k+1]}\}_{i}^{H}],
\end{aligned} \\
&
\label{eq:sigma_ma_estimation}
\begin{aligned}
\widehat{\bm \Sigma}_{\rm ma} =& \frac{1}{N_{\rm ch}}\sum_{i=1}^{N_{\rm ch}} [\{\widehat{\bm H}_{\rm ma}^{[k+1]}\}_{i}-\{{\bm \mu}_{\rm ma}^{[k+1]}\}_{i}]\cdot \\
& [\{\widehat{\bm H}_{\rm ma}^{[k+1]}\}_{i}^{H}-\{{\bm \mu}_{\rm ma}^{[k+1]}\}_{i}^{H}],
\end{aligned}
\end{align}
\end{subequations}
where
\begin{subequations}
\begin{align}
&
\label{eq:conditionalmean_ba_estimation}
\{\bm{\mu}_{\rm ba}^{[k+1]}\}_{i}=\hat{\alpha} \cdot \widehat{\bm{\Sigma}}_{\bm{H}} \widehat{\bm{\Sigma}}_{\bm{\widehat{H}}}^{-1} \{\widehat{\bm{H}}_{\rm{ba}}^{[k]}\}_{i},\\
&
\label{eq:conditionalmean_ma_estimation}
\{\bm{\mu}_{\rm ma}^{[k+1]}\}_{i}=\frac{\hat{\beta}}{\sqrt{\Theta}} \cdot \widehat{\bm{\Sigma}}_{\bm{H}} \widehat{\bm{\Sigma}}_{\bm{\widehat{H}}}^{-1} \{\widehat{\bm{H}}_{\rm{ba}}^{[k]}\}_{i}.
\end{align}
\end{subequations}

\response{In real-world environments, some practical factors limit the feasibility of the NP detector. First, the signal processing at the receiver introduces hardware-dependent noise, which prevents accurate estimation of the channel estimation noise power $\sigma^{2}$ and leads to unreliable estimates $\hat{\alpha}$, $\hat{\beta}$, $\widehat{\bm \Sigma}_{\bm{H}}$, $\widehat{\bm \Sigma}_{\rm ba}$ and $\widehat{\bm \Sigma}_{\rm ma}$. Second, LS estimation is inherently affected by AWGN, producing imperfect CSI that degrades the estimation of all the above channel statistics. Consequently, the derivation of the matrices $\bm{A}$, $\bm{B}$ and $\bm{C}$ becomes infeasible in practice, which means the performance of the theoretically optimal NP detector is often unachievable.
}

\section{LiteNP-Net, NP Detector-driven Neural Network-based PLA Design}
\label{sec:proposed}
In this section, a learning-based method driven by hypothesis testing is designed, which contains a training stage and a test stage. \response{We first present the motivation behind the proposed LiteNP-Net, then describe the LiteNP-Net design, and finally elaborate the proposed LiteNP-Net PLA system.}

\subsection{Motivation of LiteNP-Net}
\response{The motivation for designing LiteNP-Net stems from the practical inaccessibility of the coefficient matrices $\bm{A}$, $\bm{B}$, and $\bm{C}$ required by the NP detector. To overcome this limitation, the architecture of the LiteNP-Net is designed based on the mathematical model of the NP detector, in which the coefficient matrices $\bm{A}$, $\bm{B}$, and $\bm{C}$ are replaced by learnable parameters. By learning these parameters directly from CSI samples, LiteNP-Net is able to achieve performance close to the theoretical optimal NP detector without relying on channel statistics in practice.}

\subsection{LiteNP-Net Design}
\label{sec:litenp-net}
Inspired by the structure in~\eqref{eq:hypothesiscomplex}, LiteNP-Net adopts a model-driven approach to approximate the NP decision rule without relying on full channel statistics, while preserving a lightweight network structure.

\subsubsection{Design Principles}
LiteNP-Net is designed under the guidance of the mathematical model of the NP detector in~\eqref{eq:hypothesiscomplex}. However, the parameters in~\eqref{eq:hypothesiscomplex} are complex numbers and neural networks process real-valued data. Therefore, following the proof in appendix~\ref{proof:model},~\eqref{eq:hypothesiscomplex} is equivalent to
\begin{equation}
\label{eq:hypothesisrealimag}
\begin{aligned}
    & (\widehat{\bm{\mathbb{H}}}^{[k+1]})^{T} 
    \bm{A}^{\prime} 
    (\widehat{\bm{\mathbb{H}}}^{[k+1]}) + 
    (\widehat{\bm{\mathbb{H}}}_{\rm ba}^{[k]})^{T} 
    \bm{B}^{\prime}
    \widehat{\bm{\mathbb{H}}}^{[k+1]}\\
    &+(\widehat{\bm{\mathbb{H}}}_{\rm ba}^{[k]})^{T}
    \bm{C}^{\prime}
    \widehat{\bm{\mathbb{H}}}_{\rm ba}^{[k]} \underset{\mathcal{H}_1}{\overset{\mathcal{H}_0}{\gtrless}} \mathcal{T},
\end{aligned}
\end{equation}
where 
\begin{equation}
\widehat{\bm{\mathbb{H}}}_{\rm ba}^{[k]}=\begin{bmatrix} 
\Re\{ \bm{\widehat{H}}_{\rm ba}^{[k]} \}\\ 
\Im\{ \bm{\widehat{H}}_{\rm ba}^{[k]} \} 
\end{bmatrix}, \widehat{\bm{\mathbb{H}}}^{[k+1]}=\begin{bmatrix} 
\Re\{ \bm{\widehat{H}}^{[k+1]} \}\\ 
\Im\{ \bm{\widehat{H}}^{[k+1]} \} 
\end{bmatrix},
\end{equation}
with $\Im(\cdot)$ denoting the imaginary part, and 
\begin{subequations}
\label{eq:matrixABC_prime}
\begin{align}
    & \bm{A}^{\prime}=\begin{bmatrix}
    \Re\{\bm{A}\} & -\Im\{\bm{A}\} \\
    \Im\{\bm{A}\} & \Re\{\bm{A}\}
    \end{bmatrix}, \\
    & \bm{B}^{\prime}=\begin{bmatrix}
    \Re\{\bm{B}\} & -\Im\{\bm{B}\} \\
    \Im\{\bm{B}\} & \Re\{\bm{B}\}
    \end{bmatrix}, \\
    & \bm{C}^{\prime}=\begin{bmatrix}
    \Re\{\bm{C}\} & -\Im\{\bm{C}\} \\
    \Im\{\bm{C}\} & \Re\{\bm{C}\}
    \end{bmatrix}.
\end{align}
\end{subequations}

As illustrated in Fig.~\ref{fig:ProposedModel}, the LiteNP-Net consists of three embedding networks and a processing module, where $\circledast$ represents the inner product operation. The embedding networks $\Psi_{\rm A}$, $\Psi_{\rm B}$ and $\Psi_{\rm C}$ are related to the coefficient matrices $\bm{A}^{\prime}$, $\bm{B}^{\prime}$, and $\bm{C}^{\prime}$, and the processing module executes the corresponding inner product and addition operations as specified in~\eqref{eq:hypothesisrealimag}.
\begin{figure}[!t]
\centerline{\includegraphics[width=3.4in]{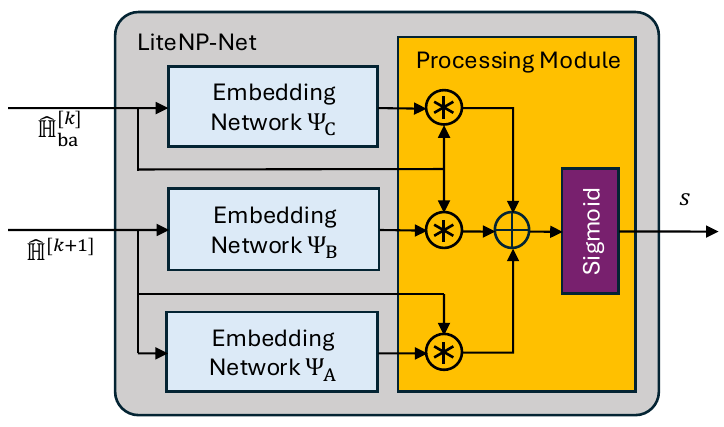}}
\caption{The structure of LiteNP-Net. }
\label{fig:ProposedModel}
\end{figure}

\subsubsection{Design Details}
Because MLPs with linear activation are equivalent to matrix multiplication, the embedding networks $\Psi_{\rm A}$, $\Psi_{\rm B}$, and $\Psi_{\rm C}$ are implemented as linear MLPs to directly represent the matrix operations of $\bm{A}^{\prime}$, $\bm{B}^{\prime}$ and $\bm{C}^{\prime}$ in~\eqref{eq:hypothesisrealimag}.
Given that the coefficient matrices $\bm{A}^{\prime}$, $\bm{B}^{\prime}$, $\bm{C}^{\prime}\in \mathbb{R}^{2M^{\prime} \times 2M^{\prime}}$, the input and output layers of embedding networks $\Psi_{\rm A}$, $\Psi_{\rm B}$ and $\Psi_{\rm C}$ are designed with $2M^{\prime}$ neurons, as shown in Fig.~\ref{fig:EmbeddingNetwork}. 
For the hidden layers, the layer with the fewest neurons is referred to as the bottleneck layer, and its number of neurons is called the latent dimension, denoted as $\mathcal{E}$, which corresponds to the rank of the coefficient matrices. An overly small latent dimension may lead to information loss during learning process, whereas an excessively large latent dimension can result in high computational overhead. Therefore, the hidden layers of embedding networks $\Psi_{\rm A}$, $\Psi_{\rm B}$ and $\Psi_{\rm C}$ are carefully designed to ensure its lightweight architecture without compromising authentication performance.
Taking embedding network $\Psi_{\rm B}$ as an example, due to the use of the linear activation function, the operation of embedding network $\Psi_{\rm B}$ can be expressed as
\begin{equation}
    \Psi_{\rm B}=\Psi_{\rm B_2} \cdot \Psi_{\rm B_1},
\end{equation}
where $\Psi_{\rm B_2}\in \mathbb{R}^{2M^{\prime} \times \mathcal{E}_{\rm B}}$ and $\Psi_{\rm B_1}\in \mathbb{R}^{\mathcal{E}_{\rm B} \times 2M^{\prime}}$, with $\mathcal{E}_{\rm B}$ representing the latent dimension of $\Psi_{\rm B}$. Based on the sub-multiplicative property of rank, it follows that 
\begin{equation}
    \text{rank}(\Psi_{\rm B}) \leq \min(\text{rank}(\Psi_{\rm B_1}), \text{rank}(\Psi_{\rm B_2}))=\mathcal{E}_{\rm B}.
\end{equation}
As demonstrated in Appendix~\ref{proof:rank}, since $\text{rank}(\bm{B}^{\prime})\leq 2L$, to ensure that $\Psi_{\rm B}$ can effectively replace $\bm{B}^{\prime}$, the latent dimension $\mathcal{E}_{\rm B}$ is designed to be $2L$. Similarly, as also proven in Appendix~\ref{proof:rank} that $\text{rank}(\bm{A}^{\prime})=2M^{\prime}$ and $\text{rank}(\bm{C}^{\prime})\leq 2L$, embedding network $\Psi_{\rm A}$ is constructed without hidden layers, indicating that it learns a matrix of rank $2M^{\prime}$, while embedding network $\Psi_{\rm C}$ follows the same structure as $\Psi_{\rm B}$, including a hidden layer with $2L$ neurons. Notably, embedding networks $\Psi_{\rm B}$ and $\Psi_{\rm C}$ do not share weights.
\begin{figure}[!t]
\centering
\subfloat[]{\includegraphics[width=1.28in]{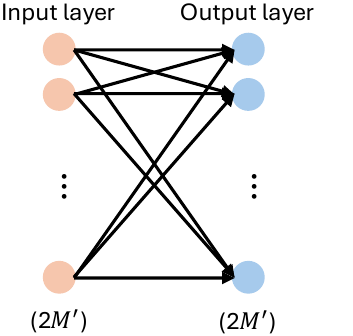}}
\subfloat[]{\includegraphics[width=1.7in]{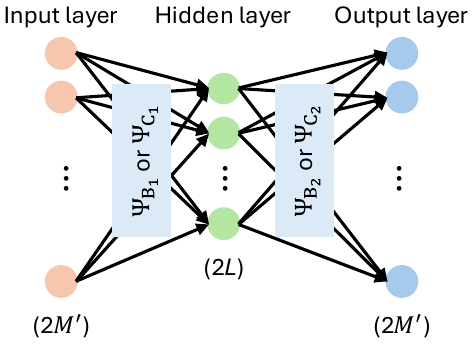}}
\caption{(a) The embedding network $\Psi_{\rm A}$. (b) The embedding networks $\Psi_{\rm B}$ and $\Psi_{\rm C}$. }
\label{fig:EmbeddingNetwork}
\end{figure}

In the processing module, the inner products between the outputs of embedding networks $\Psi_{\rm A}$, $\Psi_{\rm B}$ and $\Psi_{\rm C}$ and the corresponding input CSI measurements are computed and summed. The summed value is processed through a sigmoid activation function, generating a score $s$ between 0 and 1.

\subsection{The Proposed LiteNP-Net PLA System}
As shown in Fig.~\ref{fig:AuthenticationSystem}, the LiteNP-Net PLA  system consists of a training and a test stage. The LiteNP-Net automatically learns the channel statistics from the CSI measurements in the training stage. The trained LiteNP-Net can be utilized for device authentication in the test stage.
\begin{figure}[!t]
\centerline{\includegraphics[width=3.4in]{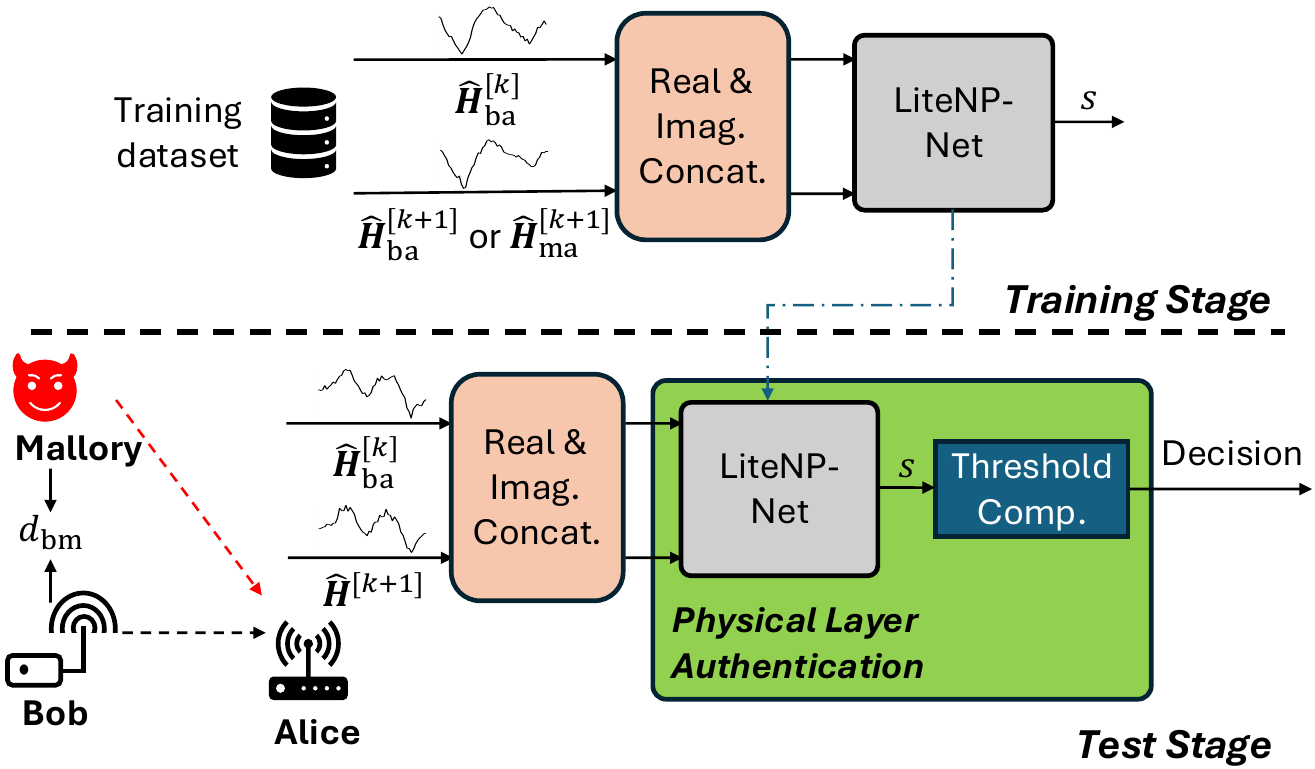}}
\caption{The proposed LiteNP-Net PLA system. }
\label{fig:AuthenticationSystem}
\end{figure}

\subsubsection{Dataset}
Assuming that the $k$-th received packet is transmitted by Bob, a set of dataset can be designed as $\mathcal{D}=\{({\bm u}_{i}, v_{i})\}_{i=1}^{N_{\text{data}}}$, where $N_{\text{data}} \in \{ N_{\text{train}}, N_{\text{test}}\}$ is the number of data samples in the training or test dataset and 
\begin{equation}
\label{eq:dataset}
({\bm u}_{i}, v_{i}) \triangleq
\begin{cases}
((\widehat{\bm H}_{\rm ba}^{[k]}, \widehat{\bm H}_{\rm ba}^{[k+1]}), 1), & \text{if}\ \widehat{\bm H}^{[k+1]}=\widehat{\bm H}_{\rm ba}^{[k+1]};\\
((\widehat{\bm H}_{\rm ba}^{[k]}, \widehat{\bm H}_{\rm ma}^{[k+1]}), 0), & \text{if}\ \widehat{\bm H}^{[k+1]}=\widehat{\bm H}_{\rm ma}^{[k+1]}.
\end{cases}
\end{equation}
We set $v_{i} = 1$ if the $(k+1)$-th received packet is also from Bob and $v_{i}=0$ if the $(k+1)$-th received packet is from Mallory. Specifically, $v_{i}$ acts as the label for the training dataset and the ground truth for the test dataset to evaluate the authentication performance of the LiteNP-Net.

\subsubsection{Training Stage}
As shown in Fig.~\ref{fig:AuthenticationSystem}, the inputs of the authentication system consist of two CSI measurements, which are processed through a real and imaginary part concatenation module before being fed into the LiteNP-Net.
It is trained to minimize the loss function defined as
\begin{equation}
\mathcal{L} =v_{i}\cdot[\max (0, \eta-s)]^2+(1-v_{i})\cdot s^2,
\end{equation}
where $\eta = 1$ is the margin parameter and $s$ denotes the output of LiteNP-Net. This training objective ensures that the output $s$ approaches 1 when the input CSI measurements originate from the same legitimate device, and tends toward 0 when they are from different devices.

\response{The architecture of the LiteNP-Net is implemented using Python 3.8 with TensorFlow and trained on NVIDIA Tesla V100 GPU. We employ the RMSprop optimizer with an initial learning rate of 0.001, a batch size of 32, and a maximum of 5000 training epochs. An early stopping criterion with a patience of 20 epochs is applied to prevent overfitting.}

\subsubsection{Test Stage} 
The trained LiteNP-Net is employed to identify rogue users. When Alice receives two CSI measurements, $\widehat{\bm H}_{\rm ba}^{[k]}$ and $\widehat{\bm H}^{[k+1]}$, she feeds them into the real and imaginary part concatenation module and then into the trained LiteNP-Net. The LiteNP-Net then generates a score $s$, which is subsequently compared against a predefined threshold to determine whether the device is legitimate or rogue, which can be given as
\begin{equation}
D=\left\{\begin{array}{rl}
    0, & \text{when } s \leq \epsilon_{\rm s}; \\
	1, & \text{when } s > \epsilon_{\rm s},
	\end{array} \right.
\end{equation}
where $D=0$ denotes that there is a rogue device, and $D=1$ denotes that there is no rogue device. The threshold $\epsilon_{\rm s}$ is obtained empirically through experiments.

\section{Benchmarks and Performance Metrics}
\label{sec:benchmark}
In this section, we first describe the benchmarks employed in this paper. Next, we present the performance metrics used in the forthcoming simulation and experimental evaluations.

\subsection{Benchmarks}
\subsubsection{Hypothesis Testing Framework}
Based on the conditional statistical model in~\cite{xiao2009channel}, it can be derived that
\begin{subequations}
\begin{align}
&
\label{eq:conditionalmean_ba_estimation_benchmark}
    \{\bm{\mu}_{\rm ba}^{[k+1]'}\}_{i}=\hat{\alpha} \cdot \{\widehat{\bm{H}}_{\rm ba}^{[k]}\}_{i},\\
&
\label{eq:conditionalmean_ma_estimation_benchmark}
    \{\bm{\mu}_{\rm ma}^{[k+1]'}\}_{i}=\frac{\hat{\beta}}{\sqrt{\Theta}} \cdot \{\widehat{\bm{H}}_{\rm ba}^{[k]}\}_{i}.
\end{align}
\end{subequations}
It can be observed that the conditional statistical model in~\cite{xiao2009channel} is a special case of~\eqref{eq:conditionalmean_ba_estimation} and~\eqref{eq:conditionalmean_ma_estimation} when channel estimation noise $\sigma^2$ is ignored, which impacts the NP detector derivation.

\subsubsection{Siamese Network-based Methods}
The works in~\cite{guo2025practical} and~\cite{zhang2025enhancing} both adopt Siamese-based approaches for device authentication, utilizing identical embedding networks to measure the similarity between channel measurements. In~\cite{guo2025practical}, the embedding network comprises two convolutional layers, a flatten layer and a fully connected layer.~\cite{zhang2025enhancing} instead employs a fully connected architecture with four layers. In both studies, device legitimacy is determined by a distance metric between the learned embeddings. As a data-driven approach, the Siamese network can capture certain channel statistics. However, since its architecture does not follow the mathematical model of the NP detector, its learning process may not align with the principles of optimal detection.

\subsubsection{Pearson Correlation-based Method}
The study in~\cite{liu2018authenticating} proposes a correlation-based approach for device authentication, where the Pearson correlation coefficient is utilized to measure the similarity between two consecutive CSI measurements, $\widehat{\bm H}_{\rm ba}^{[k]}$ and $\widehat{\bm H}^{[k+1]}$. A decision can be made by comparing the correlation with an empirically determined threshold. Despite its simplicity, the correlation-based method lacks a learning process and relies entirely on raw CSI measurements, leading to reduced robustness in practical applications.

\subsection{Performance Metrics}
The true positive rate (TPR), false positive rate (FPR), receiver operating characteristics (ROC) and area under the curve (AUC) are used as the performance metrics. Specifically, the TPR quantifies the proportion of actual positive cases that are correctly identified by the LiteNP-Net. On the other hand, the FPR measures the proportion of negative cases that are mistakenly classified as positive. The ROC curve is constructed using TPR and FPR, illustrating the trade-off between them across different threshold settings. Additionally, the AUC represents the area under ROC, providing a summary of the discriminative ability.

To evaluate the computational overhead of neural networks, parameter count and floating point operations (FLOPs) are used as performance metrics. The parameter count indicates the total number of trainable parameters within the model, which are optimized during training to minimize the loss function. This metric can be obtained using the built-in model.summary() function in TensorFlow. FLOPs measures the total number of floating point operations required for a single forward pass through the network and can be computed using the built-in tf.compat.v1.profiler() funtion in TensorFlow.

\section{Simulation Evaluation}
\label{sec:simulation}
In this section, we evaluate the performance of the LiteNP-Net in a controlled yet realistic channel environment. In the simulation, the channel parameters are precisely controlled, which allows us to evaluate the performance of the theoretically optimal NP detector and compare it with LiteNP-Net. In this section, the simulation setup is first introduced, followed by a detailed analysis of the simulation results.

\subsection{Simulation Setup}
To capture diverse WLAN conditions, a set of environment-specific channel models was used, as summarized in Table~\ref{tab:WLANchannelmodels}. Each model is defined by a power delay profile (PDP) composed of multiple clusters with exponential decay. Detailed PDP values can be found in~\cite{Erceg2004IEEEPW}.
\begin{table}[!t]
    \centering
    \caption{WLAN TGn Channel Models~\cite{Erceg2004IEEEPW}.}
    \begin{tabular}{cccc}
         \hline
         Model & RMS Delay (ns)& No. Clusters&Mapped Environment\\
         \hline
         Model B&15&2& Residential apartment\\
         Model C&30&2& Small office\\
         Model D&50&3& Typical office\\
         Model E&100&4& Large office\\
         Model F&150&6& Large space\\
         \hline
    \end{tabular}
    \label{tab:WLANchannelmodels}
\end{table}
According to~\cite{Erceg2004IEEEPW}, the Doppler spectrum follows a Bell-shaped distribution, given in the linear scale as
\begin{equation}
S(f)=\frac{\sqrt{A}}{\pi f_{\rm d}} \cdot \frac{1}{1+A(\frac{f}{f_{\rm d}})^2},
\end{equation}
where $A$ is a constant 9. Therefore, the autocorrelation function can be calculated as
\begin{equation}
R(\Delta t)=e^{-\frac{2 \pi f_{\rm d}}{\sqrt{A}}\Delta t},
\end{equation}
where $\Delta t$ represents the time interval.

The training and test simulation datasets can be denoted as $\mathcal{D}_{\text{train}}^{\rm S}=\{({\bm u}_{i}^{\rm S}, v_{i}^{\rm S})\}_{i=1}^{N_{\text{train}}^{\rm S}}$ and $\mathcal{D}_{\text{test}}^{\rm S}=\{({\bm u}_{j}^{\rm S}, v_{j}^{\rm S})\}_{j=1}^{N_{\text{test}}^{\rm S}}$.
Specifically, the channel ${\bm h}_{\rm ba}^{[k]}$ was generated using the IEEE 802.11 TGn channel models via the MATLAB WLAN toolbox\footnote{https://www.mathworks.com/help/wlan/ref/wlantgnchannel-system-object.html}. To model temporal and spatial evolution, the subsequent channel ${\bm h}_{\rm ba}^{[k+1]}$ and ${\bm h}_{\rm ma}^{[k+1]}$ were derived based on correlation models provided in~\eqref{eq:channelcorrelation_time}.
A WLAN Non-HT waveform was then created and transmitted through ${\bm h}_{\rm ba}^{[k]}$, ${\bm h}_{\rm ba}^{[k+1]}$ and ${\bm h}_{\rm ma}^{[k+1]}$, respectively, and the AWGN was added at the receiver. The receiver extracted the LTS from the received waveforms to estimate the CSIs, yielding $\widehat{\bm H}_{\rm ba}^{[k]}$, $\widehat{\bm H}_{\rm ba}^{[k+1]}$, and $\widehat{\bm H}_{\rm ma}^{[k+1]}$, which were utilized to construct the simulation datasets as defined in~\eqref{eq:dataset}. In the simulation, $N_{\text{train}}^{\rm S}=50,000$ and $N_{\text{test}}^{\rm S}=10,000$.

\subsection{Simulation Results}
Before conducting the performance analysis of the LiteNP-Net, firstly, we determined the suitable latent dimension $\mathcal{E}_{\rm B}$ and $\mathcal{E}_{\rm C}$ for the embedding networks $\Psi_{\rm B}$ and $\Psi_{\rm C}$. As demonstrated in Section~\ref{sec:proposed}, embedding network $\Psi_{\rm A}$ does not contain hidden layers, while the latent dimension of embedding networks $\Psi_{\rm B}$ and $\Psi_{\rm C}$ cannot be smaller than $2L$. Because the WLAN TGn channel model F has the longest RMS delay among the WLAN TGn channel models, it has higher requirements for the latent dimension compared to other models. Therefore, we examined the latent dimension in embedding networks $\Psi_{\rm B}$ and $\Psi_{\rm C}$ on the model F. \response{Fig.~\ref{fig:AUCvsTsimulation} shows the AUC versus $\Delta t_{k}$ using LiteNP-Net with different latent dimensions. The performance of the NP detector was obtained through~\eqref{eq:hypothesiscomplex}, where the channel statistics were estimated and computed by~\eqref{eq:alpha_eatimation}-\eqref{eq:conditionalmean_ma_estimation}.
It can be observed that the theoretically optimal NP detector achieves the best performance. However, it is important to note that this performance can only be achieved when the receiver has complete knowledge of the channel statistics, which is often unattainable in practical applications. Therefore, the performance of NP detector only serves as an optimal upper bound. Additionally, with the latent dimension below 32, the performance of the LiteNP-Net deteriorates due to the insufficient learning of channel statistics by the embedding networks $\Psi_{\rm B}$ and $\Psi_{\rm C}$. However, once the latent dimension reaches 32, 64, and 104, the performance of the LiteNP-Net closely approaches the theoretically optimal NP detector. This is because a large latent dimension enables the embedding networks $\Psi_{\rm B}$ and $\Psi_{\rm C}$ to effectively capture the channel statistics. Given the computational overhead, the embedding network with a latent dimension of 32 is chosen for further evaluation.}

\begin{figure}[!t]
\centerline{\includegraphics[width=3.4in]{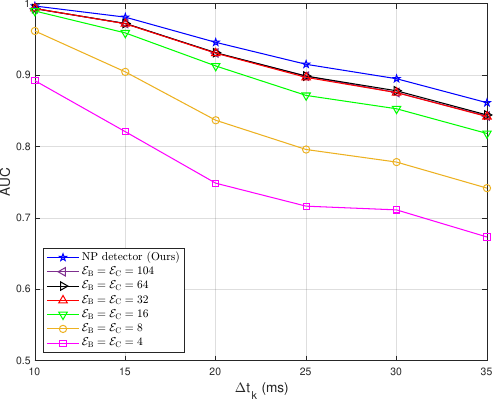}}
\caption{The AUC versus $\Delta t_{k}$ on the simulation test dataset with WLAN TGn channel model F, $\text{SNR}=6$~dB, $d_{\rm bm}/\lambda=0.25$ and $v_0=1$ m/s.}
\label{fig:AUCvsTsimulation}
\end{figure}

\response{Fig.~\ref{fig:AUCvsTablation} presents the ablation results of different neural network architectures and training losses. It can be observed that LiteNP-Net closely approaches the theoretically optimal NP detector. However, keeping two embedding networks, i.e., $\Psi_{\rm B}+\Psi_{\rm C}$, $\Psi_{A}+\Psi_{\rm C}$, and $\Psi_{\rm A}+\Psi_{\rm B}$, leads to the degraded AUC, indicating that all three embedding networks contribute essential channel statistics for robust PLA. Moreover, the choice of training loss, i.e., contrastive loss and binary cross-entropy, has comparatively minor impact on performance.}
\begin{figure}[!t]
\centerline{\includegraphics[width=3.4in]{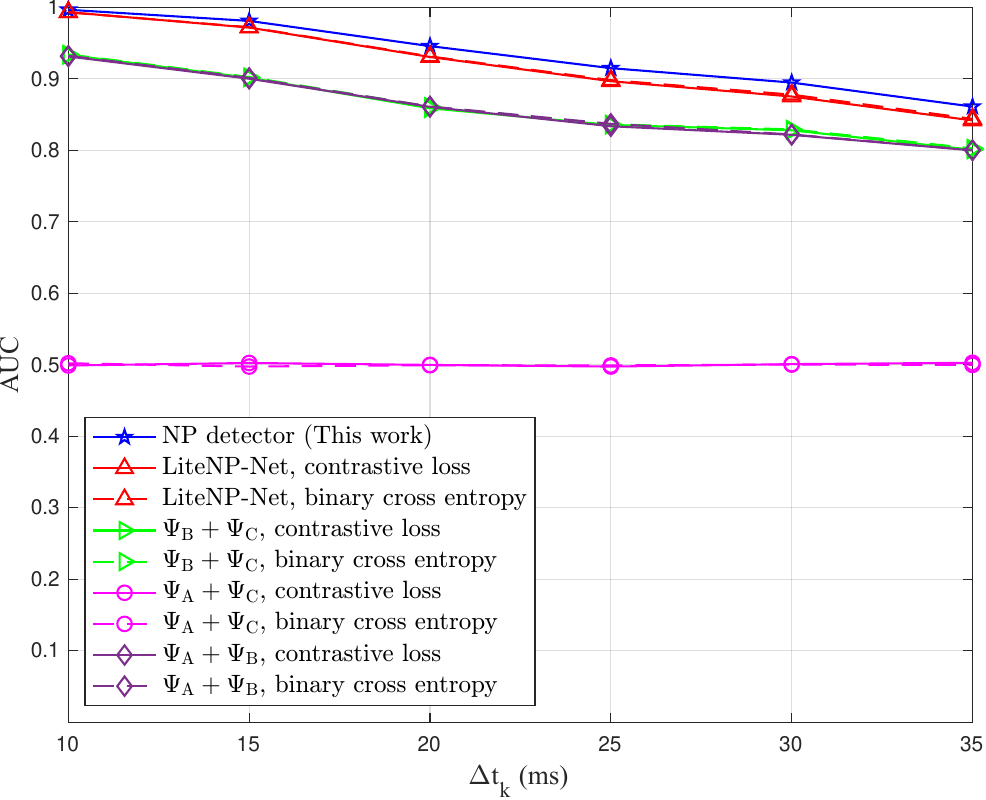}}
\caption{The ablation comparison on the simulation test dataset with WLAN TGn channel model F, $\text{SNR}=6$~dB, $d_{\rm bm}/\lambda=0.25$ and $v_0=1$ m/s.}
\label{fig:AUCvsTablation}
\end{figure}

In the subsequent performance analysis, we evaluated LiteNP-Net with the WLAN TGn channel model B. Due to its short RMS delay, model B exhibits weak frequency selectivity in the CSI, making it more difficult for LiteNP-Net to extract channel features from CSI measurements. Accordingly, model B serves as a relatively challenging case among the WLAN TGn channel models.
Fig.~\ref{fig:ROCsimulation} depicts the ROC curves of our NP detector in~\eqref{eq:hypothesiscomplex}, LiteNP-Net, FCN-based Siamese model in~\cite{zhang2025enhancing}, CNN-based Siamese model in~\cite{guo2025practical}, Pearson correlation-based method in~\cite{liu2018authenticating} and the NP detector derived from~\cite{xiao2009channel}.
\response{It can be seen that the LiteNP-Net can closely approach the performance of our NP detector. This is because the LiteNP-Net is specifically designed to follow the structure of the NP detector in~\eqref{eq:hypothesiscomplex} and it automatically learns the channel statistics from CSI measurements during the training stage. On the other hand, the FCN-based Siamese model, CNN-based Siamese model, and Pearson correlation-based method rely on a heuristic similarity metric between CSI measurements, without any guarantee that it matches the theoretically optimal decision rule. As a result, they may not fully exploit the statistical structure of wireless channels and thus yield suboptimal performance.} Additionally, since the conditional statistical model in~\cite{xiao2009channel} does not take channel estimation noise into account, the derived NP detector exhibits significantly degraded performance.
\begin{figure}[!t]
\centerline{\includegraphics[width=3.4in]{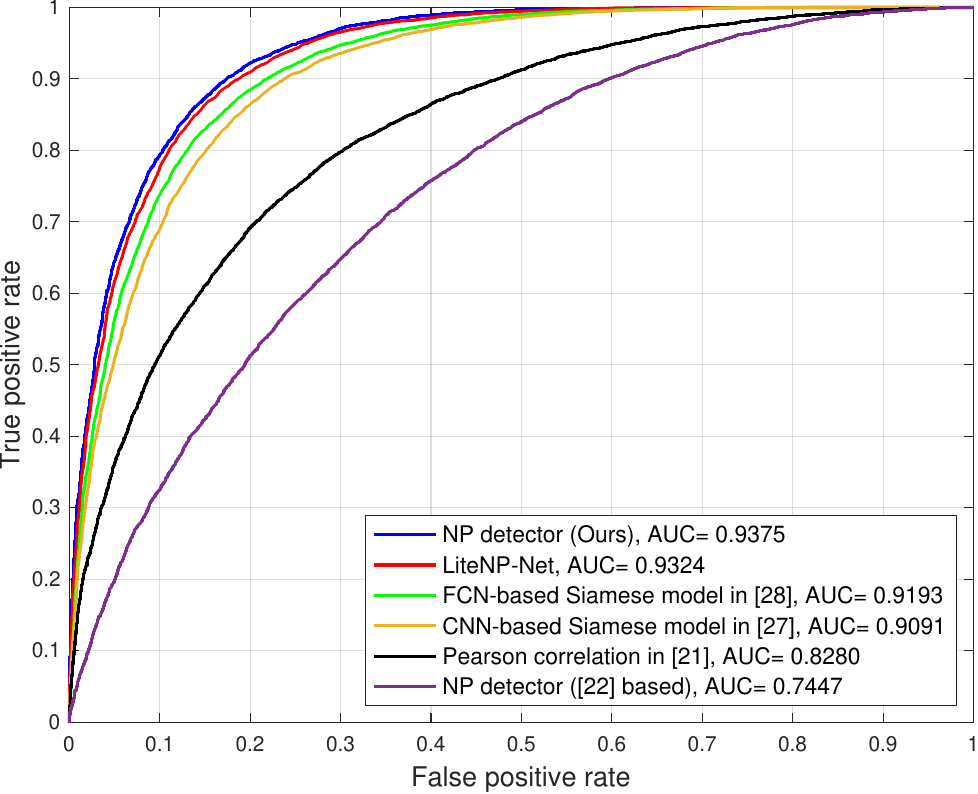}}
\caption{ROC curves on simulation test dataset with WLAN TGn channel model B, $\text{SNR}=12$~dB, $\Delta t_k=20$~ms, $d_{\rm bm}/\lambda=1$ and $v_0=1$~m/s.}
\label{fig:ROCsimulation}
\end{figure}

Fig.~\ref{fig:AUCvsSNRsimulation} illustrates the AUC of our approaches and the benchmarks versus SNR. The LiteNP-Net achieves performance close to the theoretical upper bound obtained by our NP detector in~\eqref{eq:hypothesiscomplex} and outperforms all the benchmark methods. As the estimation quality of CSI measurements improves at higher SNR levels, the performance of all methods improves with the increasing SNR.
\begin{figure}[!t]
\centerline{\includegraphics[width=3.4in]{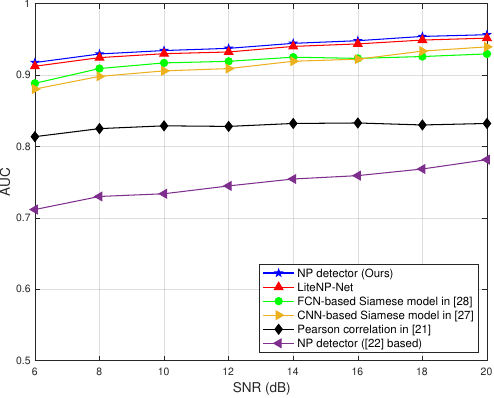}}
\caption{The AUC versus SNR on the simulation test dataset with WLAN TGn channel model B, $d_{\rm bm}/\lambda=1$, $\Delta t_{k}=20$ ms and $v_0=1$ m/s.}
\label{fig:AUCvsSNRsimulation}
\end{figure}

Fig.~\ref{fig:AUCvsDsimulation} shows the AUC versus the distance between Bob and Mallory normalized by wavelength $d_{\rm bm}/\lambda$. It can be illustrated that LiteNP-Net surpasses all benchmark methods in performance. Additionally, as the distance between Bob and Mallory increases, Alice performs better, achieving a higher AUC in detecting rogue devices.
\begin{figure}[!t]
\centerline{\includegraphics[width=3.4in]{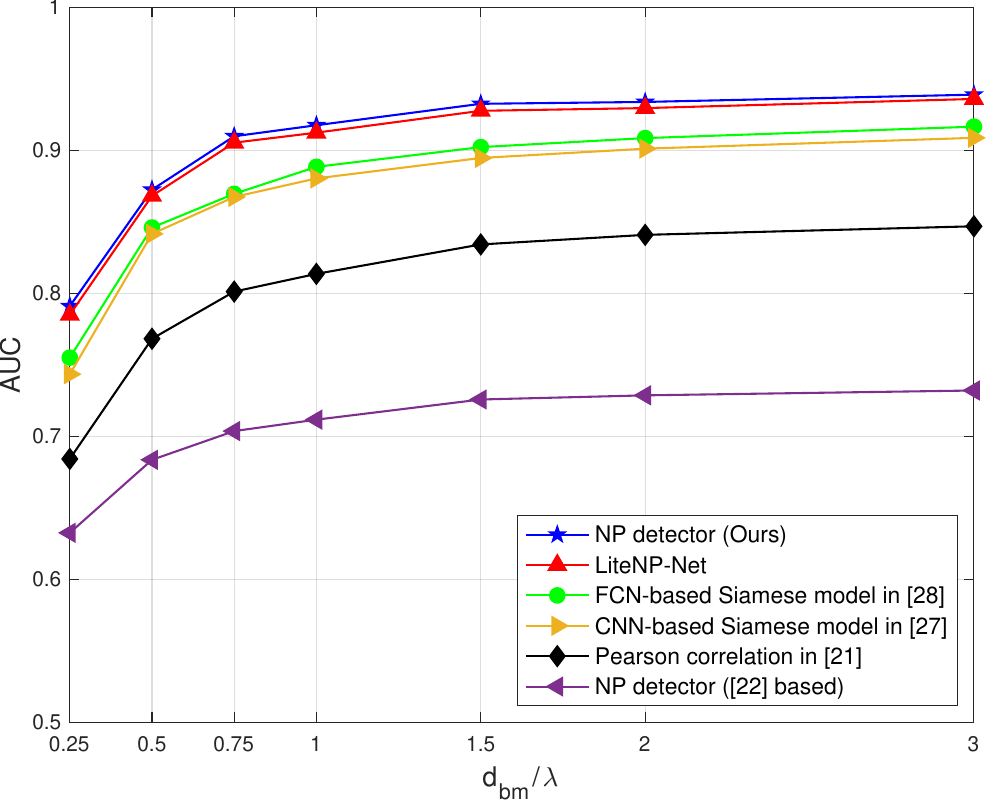}}
\caption{The AUC versus $d_{\rm bm}/\lambda$ on the simulation test dataset with WLAN TGn channel model B, $\text{SNR}=6$~dB, $\Delta t_{k}=20$ ms and $v_0=1$ m/s.}
\label{fig:AUCvsDsimulation}
\end{figure}

\response{Table~\ref{tab:computationalcomparison} presents a comparison of the computational overhead in terms of parameter count, FLOPs and inference time among LiteNP-Net, CNN-based Siamese model in~\cite{guo2025practical} and FCN-based Siamese model in~\cite{zhang2025enhancing}.} The computational overhead of the NP detector and the Pearson correlation-based method is not included, as they are non-learning-based approaches and do not involve trainable parameters or significant model inference costs. \response{Because Siamese networks do not integrate communication domain knowledge, they are not specifically optimized and consequently leads to larger parameter count, higher FLOPs and longer inference time. In contrast, LiteNP-Net is explicitly driven by the mathematical model of the NP detector, where each embedding network is implemented as a lightweight MLP network with least parameters to approximate the coefficient matrices in the NP detector. Thanks to the model-driven approach, LiteNP-Net achieves a significant reduction in computational overhead compared to both the CNN-based Siamese model and FCN-based Siamese model.
\begin{table}[!t]
    \centering
    \caption{Comparison of Computational Overhead.}
    \begin{tabular}{cccc}
         \hline
         & Parameter & FLOPs & Inference Time\\
         \hline
         LiteNP-Net&25,258&433,502&1.016 ms\\
         CNN-based Siamese in~\cite{guo2025practical}&57,108&2,141,480&1.390 ms\\
         FCN-based Siamese in~\cite{zhang2025enhancing}&296,209&1,558,516&1.076 ms\\
         \hline
    \end{tabular}
    \label{tab:computationalcomparison}
\end{table}
}

\section{Experimental Evaluation}
\label{sec:experimental}
In this section, we further explore the performance of the LiteNP-Net in experimental indoor wireless environments. The experimental setup is introduced first, followed by the evaluation of the experimental results.

\subsection{Experiment Setup}
\subsubsection{Device Configuration}
As illustrated in Fig.~\ref{fig:ExperimentDevices}, Alice was configured as a Wi-Fi access point (AP), while Bob and Mallory acted as stations. Alice was implemented using an ESP32 kit, which comes with built-in Wi-Fi connectivity. The ESP32 microcontroller gathered CSIs through the ESP32 CSI Toolkit\footnote{https://github.com/StevenMHernandez/ESP32-CSI-Tool}, a tool that provides detailed information, including the operating mode, MAC address of the transmitter, received signal strength indicator (RSSI), noise floor, time stamp of the received signal, and CSI. Once collected, the data was transmitted to a personal computer (PC) via a universal serial bus (USB) connection for further processing. In addition to Alice, two LoPy4 development boards\footnote{https://development.pycom.io/tutorials/networks/wlan/} were used to represent Bob and Mallory. These boards functioned in Wi-Fi station mode, transmitting signals to Alice. All the devices were configured to support IEEE 802.11n, operating with a $20$ MHz bandwidth and a $2.4$ GHz carrier frequency.
\begin{figure}[!t]
\centerline{\includegraphics[width=3in]{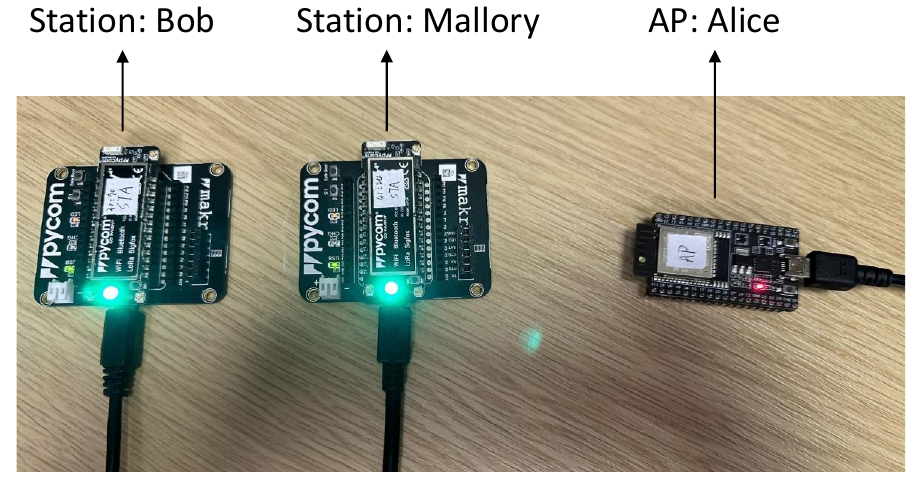}}
\caption{Experiment devices~\cite{guo2025practical}. (i) ESP32 in the Wi-Fi AP mode was used as Alice, and (ii) two LoPy4 boards in the Wi-Fi station mode were used as Bob and Mallory.}
\label{fig:ExperimentDevices}
\end{figure}

\subsubsection{Experiment Scenarios}
To evaluate the accuracy of the LiteNP-Net, four distinct scenarios were considered.
\begin{itemize}
    \item Scenario I: movement along a corridor A, as shown in Fig.~\ref{fig:TestScenarios}(a).
    \item Scenario II: movement within an office, as shown in Fig.~\ref{fig:TestScenarios}(b) (moving route 1).
    \item Scenario III: movement along a corridor B, as shown in Fig.~\ref{fig:TestScenarios}(b) (moving route 2).
    \item Scenario IV: movement in a residential apartment (floor plan not provided).
\end{itemize}
The experiments for scenario I and scenarios II \& III were conducted on the second and sixth floors of the Department of Electrical Engineering and Electronics at the University of Liverpool, UK. Notably, the datasets for training and test were independently gathered for each of the scenarios.
\begin{figure}[!t]
\centering
\subfloat[]{\includegraphics[width=3.2in]{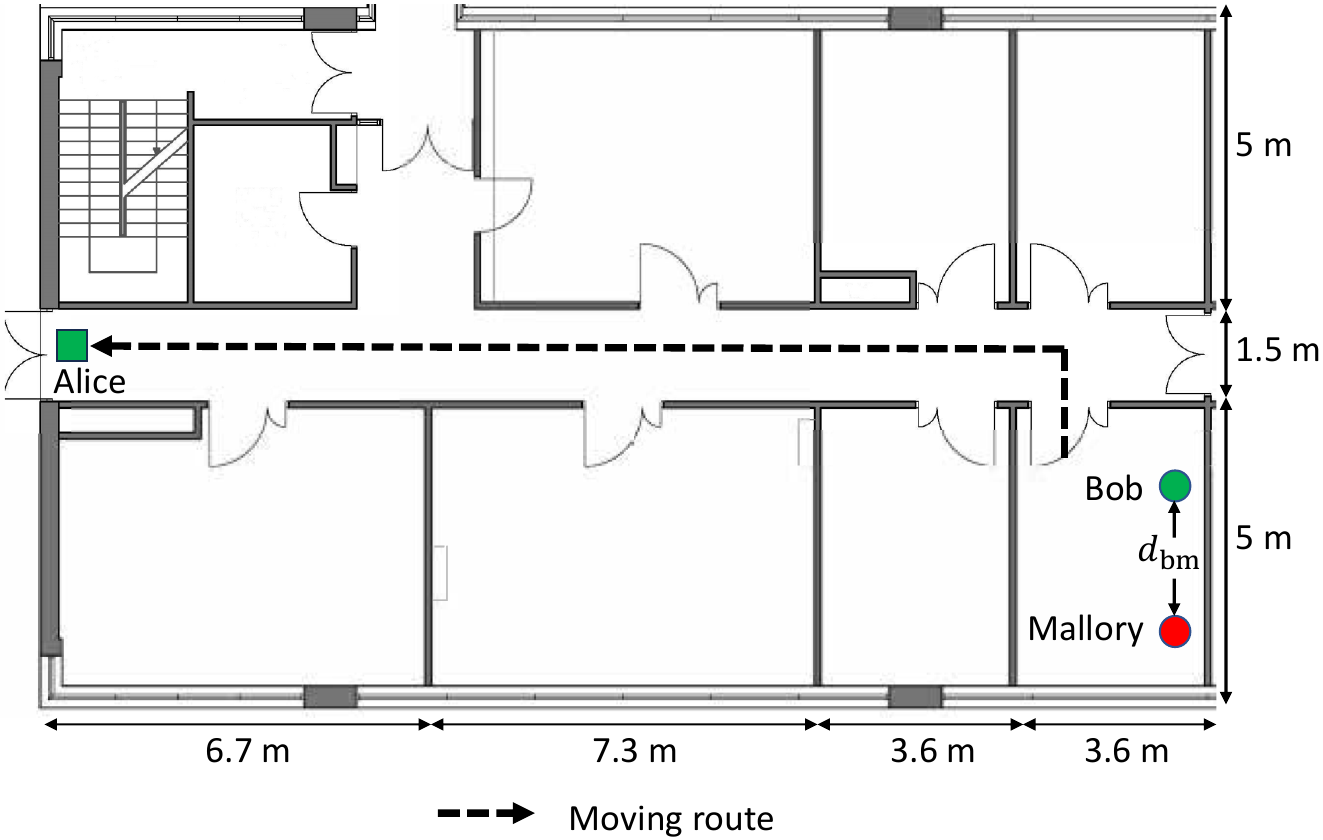}}

\subfloat[]{\includegraphics[width=3.4in]{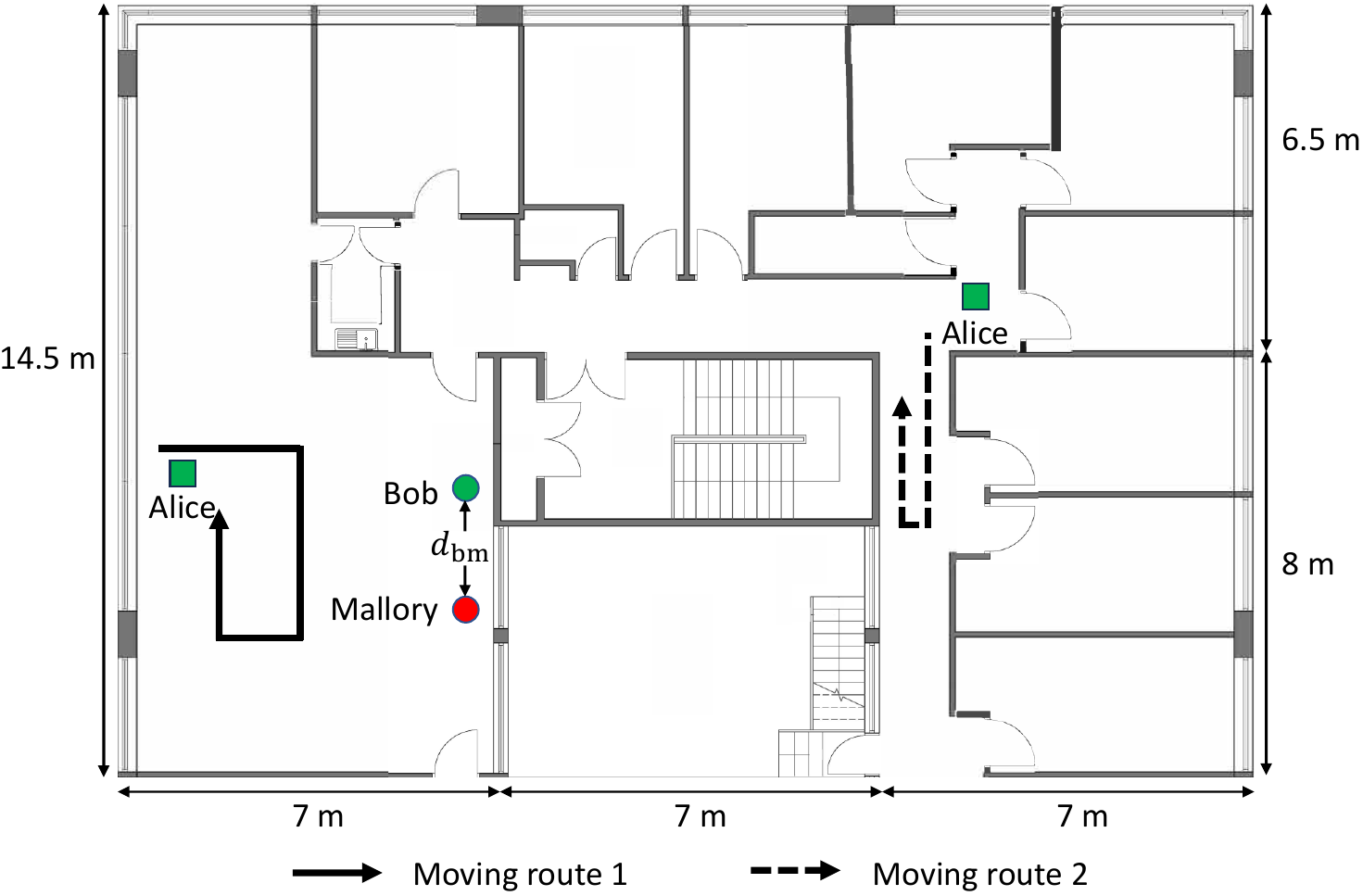}}

\caption{Experiment scenarios~\cite{guo2025practical}. (a) Scenario I: corridor A. (b) Scenario II: office (Moving route 1), and Scenario III: corridor B (Moving route 2). }
\label{fig:TestScenarios}
\end{figure}

\subsubsection{Experimental Training Dataset Collection}
Due to the non-cooperative nature of the attacker during the training stage, the training datasets were collected between Alice and Bob across all experimental scenarios. \response{In the experiments, Alice was carried by an experimenter and held at a height of approximately 1.2 m while moving along the predefined routes shown in Fig.~\ref{fig:TestScenarios} at a constant speed of 0.25 m/s, transmitting packets every 0.01 s. Bob was placed on a desk at the same height of approximately 1.2 m and remained stationary, with CSI measurements collected using the ESP32 CSI toolkit.}

As Alice moved, the CSI measurements collected at different time instants correspond to transmissions from different locations. Consequently, CSI measurements with significant time gaps typically represent transmissions from the device at widely separated locations, which can be interpreted as coming from different devices. Therefore, in our experimental training dataset $\mathcal{D}^{\rm E}_{\text{train}}=\{({\bm u}_{i}^{\rm E},v_{i}^{\rm E})\}_{i=1}^{N_{\text{train}}^{\rm E}}$, each input ${\bm u}_{i}^{\rm E}\triangleq (\widehat{\bm H}^{[k]},\widehat{\bm H}^{[k+\Delta k]})$ represents a pair of channel measurements separated by $\Delta k$ packets. When $\Delta k = 1$, we assign $v_{i}^{\rm E} = 1$, indicating that the CSI measurements are from the same device. When $\Delta k = 50$, we assign $v_{i}^{\rm E} = 0$, indicating that the measurements are from different devices. \response{In total, we collected $5145$, $5037$, $5077$, and $4240$ CSI measurements in scenario I, scenario II, scenario III, and scenario IV, respectively.}

\subsubsection{Experimental Test Dataset Collection}
The experimental test dataset was collected with a focus on accurately controlling and adjusting the distance between Bob and Mallory. To achieve this, we opted for the movement of the AP instead of the user stations. Due to channel reciprocity, this setup is equivalent to having a fixed AP with mobile user stations.
\response{Throughout all scenarios, Bob and Mallory were positioned on a desk at a height of approximately 1.2 m, with their separation distance $d_{\rm bm}$ precisely controlled along a marked straight line and set to 3, 6, 9, 12, 18, 24, and 36 cm. Given the operating frequency of 2.4 GHz, where the signal wavelength is about 12 cm, these distances correspond to 0.25, 0.5, 0.75, 1, 1.5, 2, and 3 wavelengths, respectively. Bob and Mallory transmitted signals to Alice at intervals of 0.01 s and 0.1 s, respectively. The movement pattern of Alice remained the same as in the training stage. For each attack-distance configuration, we collected 2000, 2800, 2700, and 2500 CSI measurements for scenario I, scenario II, scenario III, and scenario IV, respectively.} 
Furthermore, by extracting the MAC addresses of the transmitters, Alice was able to accurately assign the correct identity to each received packet.

\subsubsection{Preprocessing of Experimental Dataset}
In the experiments, phase noise, caused by imperfections in hardware oscillators, affects CSI measurements. To enhance authentication reliability, we applied the phase noise compensation method from~\cite{xiao2009channel} to preprocess the data before training LiteNP-Net.

The experimental dataset before preprocessing can be found in~\cite{dataset_csi}.

\subsection{Experiment Results}
Fig.~\ref{fig:ROCexperiment} shows the ROC curves of the LiteNP-Net, CNN-based Siamese model in~\cite{guo2025practical}, FCN-based Siamese model in~\cite{zhang2025enhancing} and Pearson correlation-based method in~\cite{liu2018authenticating} on the test dataset collected in the scenario I (corridor A). 
\begin{figure}[!t]
\centerline{\includegraphics[width=3.4in]{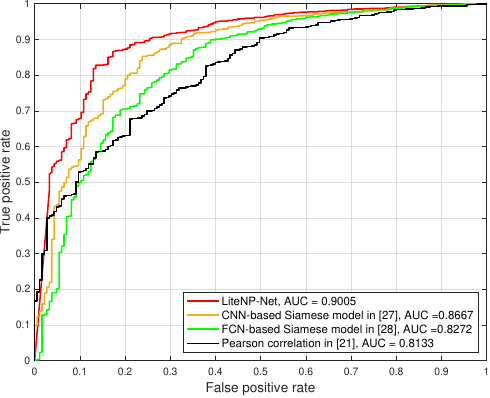}}
\caption{ROC curves on the test dataset collected in the scenario I (corridor A) with $d_{\rm bm}/\lambda=0.25$.}
\label{fig:ROCexperiment}
\end{figure}
In the experimental scenario, it is unfeasible to obtain the channel statistics required for NP detection, making it impossible to achieve the theoretically optimal performance. However, when applying the LiteNP-Net, it can automatically learn the coefficient matrix $\bm{A}$, $\bm{B}$, and $\bm{C}$ from collected CSI measurements during the training stage. In the test stage, the LiteNP-Net can potentially function as an NP detector. In contrast, the similarity-based methods serve as suboptimal solutions. Therefore, the LiteNP-Net outperforms CNN-based Siamese model, FCN-based Siamese model and Pearson correlation-based method, achieving a higher AUC.

Fig.~\ref{fig:AUC_scenario} illustrates the AUC versus the distance between Bob and Mallory normalized by wavelength $d_{\rm bm}/\lambda$ in various experimental scenarios. The LiteNP-Net, CNN-based Siamese model and FCN-based Siamese model are trained and tested in separate experimental scenarios. It can be observed that the LiteNP-Net outperforms the correlation-based method in all scenarios and surpasses the CNN-based Siamese model and FCN-based Siamese model in most scenarios.
\begin{figure}[!t]
\centering
\subfloat[Scenario I.]{\includegraphics[width=3.4in]{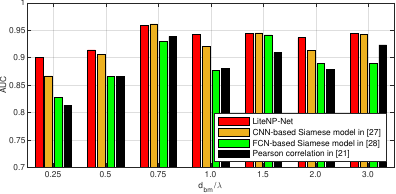}}

\subfloat[Scenario II.]{\includegraphics[width=3.4in]{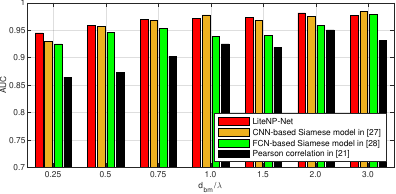}}

\subfloat[Scenario III.]{\includegraphics[width=3.4in]{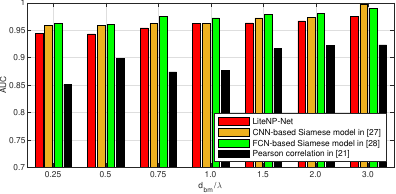}}

\subfloat[Scenario IV.]{\includegraphics[width=3.4in]{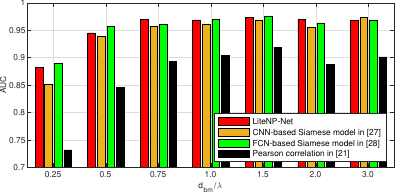}}
\caption{The AUC versus the distance between Bob and Mallory normalized by wavelength $d_{\rm bm}/\lambda$ in various experimental scenarios.}
\label{fig:AUC_scenario}
\end{figure}

\section{Discussion}
\label{sec:discussion}
\subsection{Threshold Selection}
\response{In the NP framework, the threshold $\mathcal{T}$ is determined by enforcing a target false-alarm probability $P_{\rm FA}$. Specifically, the threshold is chosen such that
\begin{equation}
    P_{\rm FA} = \Pr (\Lambda < \mathcal{T} \mid \mathcal{H}_0 ),
\end{equation}
where $\Lambda$ is defined in \eqref{eq:hypothesiscomplex}. Since the distribution of $\Lambda$ under $\mathcal{H}_0$ depends on the channel statistics, computing threshold requires knowledge of the channel distribution of the environment. Consequently, the threshold generally needs to be re-tuned when the environment changes. In experimental settings, where the channel distribution is typically unknown, the threshold $\epsilon_{s}$ is selected empirically from the ROC curve based on the desired trade-off between false-alarm probability and detection performance.}

\subsection{Potential Limitations}
\response{While LiteNP-Net demonstrates strong performance across both simulation and experimental evaluations, there are some potential limitations.

\subsubsection{Channel Model Mismatch}
LiteNP-Net is trained to capture the underlying channel statistics and then applies the learned decision structure during testing. This implicitly assumes that the statistical properties of the channel remain sufficiently consistent between the training and test stages. If the channel distribution during testing differs significantly from that observed during training, for example due to changes in the environmental layout or mobility patterns, the generalization performance of LiteNP-Net may deteriorate.

\subsubsection{Advanced Adversarial Strategies}
This paper assumes that the attacker does not try to replicate the legitimate user’s CSI. More sophisticated adversaries may imitate temporal CSI patterns or probe the environment to approximate the legitimate user’s channel. Countermeasures should be designed to resist these stronger attacks.}

\section{Conclusion}
\label{sec:conclusion}
In this paper, we presented a learning-based PLA scheme driven by hypothesis testing, and carried out thorough simulation and experimental evaluation. Specifically, we integrated a practical conditional statistical model into the hypothesis testing framework, enabling the derivation of a theoretically optimal NP detector. A LiteNP-Net driven by hypothesis testing was designed to ensure its lightweight architecture without compromising authentication performance. A notable aspect of this paper is the approach incorporating both simulation and experimental evaluation. In controlled simulation evaluations, the performance of the theoretically optimal NP detector can be calculated, allowing us to compare LiteNP-Net with the derived NP detector. \response{Furthermore, comprehensive ablation studies were conducted on the latent dimensions of the embedding networks, the architectures of the neural networks, and the choice of training loss. We also examined the impact of transmission intervals, SNR, and attack distances on authentication performance.} To further assess the reliability of the LiteNP-Net, we deployed an experimental testbed using Wi-Fi IoT development kits in typical indoor environments, including both LOS and NLOS conditions. Demonstrated by simulation and experimental results, the LiteNP-Net outperformed the traditional correlation-based method and state-of-the-art Siamese-based methods.

\appendices
\section{The proof of the NP detector in~\eqref{eq:hypothesiscomplex}}
\label{proof:pdf}
\response{
\begin{proof}
To facilitate the derivation of the likelihood function in hypothesis testing, we first present a lemma regarding the posterior distribution of $\bm{H}_{\rm{ba}}^{[k]}$ given its measurement $\widehat{\bm{H}}_{\rm{ba}}^{[k]}$, which is derived based on the properties of conditional distribution of jointly Gaussian variables~\cite[Chapter 2]{bishop2006pattern}.
\begin{lemma}
Given the channel estimation model in~\eqref{eq:lsEstimation_vector}, the posterior distribution of $\bm{H}_{\rm{ba}}^{[k]}$ given its measurement $\widehat{\bm{H}}_{\rm{ba}}^{[k]}$ can be expressed as
\begin{equation}
\label{eq:posteriordistribution}
\bm{H}_{\rm{ba}}^{[k]} \mid \widehat{\bm{H}}_{\rm{ba}}^{[k]} \sim \mathcal{C N}\left(\bm{\mu}_{\bm{H} \mid \widehat{\bm{H}}}, \bm{\Sigma}_{\bm{H} \mid \widehat{\bm{H}}}\right),
\end{equation}
where 
\begin{subequations}
\begin{align}
\bm{\mu}_{\bm{H} \mid \widehat{\bm{H}}}&=\mathbb{E}[\bm{H}_{\rm ba}^{[k]} \mid \bm{\widehat{H}}_{\rm ba}^{[k]}] 
=\bm{\Sigma}_{\bm{H}} \bm{\Sigma}_{\bm{\widehat{H}}}^{-1} \widehat{\bm{H}}_{\rm{ba}}^{[k]}, \\
\bm{\Sigma}_{\bm{H} \mid \widehat{\bm{H}}}&=\operatorname{Cov}(\bm{H}_{\rm{ba}}^{[k]} \mid \widehat{\bm{H}}_{\rm{ba}}^{[k]})
=\bm{\Sigma}_{\bm{H}}-\bm{\Sigma}_{\bm{H}} \bm{\Sigma}_{\bm{\widehat{H}}}^{-1} \bm{\Sigma}_{\bm{H}}.
\end{align}
\end{subequations}
\end{lemma}

Under null hypothesis $\mathcal{H}_0$, the conditional distribution of $\bm{\widehat{H}}_{\rm ba}^{[k+1]}$ given $\bm{\widehat{H}}_{\rm ba}^{[k]}$ is 
\begin{equation}
    \bm{\widehat{H}}_{\rm ba}^{[k+1]} \mid \bm{\widehat{H}}_{\rm ba}^{[k]} \sim \mathcal{CN}(\bm{\mu}_{\rm ba}^{[k+1]}, \bm{\Sigma}_{\rm ba}^{[k+1]}),
\end{equation}
where conditional mean
\begin{equation}
\label{eq:conditionalmean_ba}
\begin{aligned}
    \bm{\mu}_{\rm ba}^{[k+1]} =& \mathbb{E}\big[\bm{\widehat{H}}^{[k+1]}=\bm{\widehat{H}}_{\rm ba}^{[k+1]} \mid \bm{\widehat{H}}^{[k]}=\bm{\widehat{H}}_{\rm ba}^{[k]} \big]\\
    =& \mathbb{E} \big[(\alpha \bm{H}_{\rm ba}^{[k]}+\sqrt{1-\alpha^2} \bm{\Omega}_{1}^{[k+1]}+\bm{\widehat{Z}}) \mid \bm{\widehat{H}}_{\rm ba}^{[k]} \big] \\
    =& \alpha \cdot \mathbb{E} \big[\bm{H}_{\rm ba}^{[k]} \mid \bm{\widehat{H}}_{\rm ba}^{[k]} \big] + \sqrt{1-\alpha^2} \cdot \mathbb{E} \big[\bm{\Omega}_{1}^{[k+1]} \big] + \mathbb{E} \big[ \bm{\widehat{Z}} \big]\\
    =& \alpha \cdot \mathbb{E} \big[\bm{H}_{\rm ba}^{[k]} \mid \bm{\widehat{H}}_{\rm ba}^{[k]} \big] \\
    =& \alpha \cdot \bm{\Sigma}_{\bm{H}} \bm{\Sigma}_{\bm{\widehat{H}}}^{-1} \widehat{\bm{H}}_{\rm{ba}}^{[k]},
\end{aligned}
\end{equation}
and conditional covariance
\begin{equation}
\label{eq:covariance_ba}
\begin{aligned}
    \bm{\Sigma}_{\rm ba}^{[k+1]} =& \operatorname{Cov} \big[\bm{\widehat{H}}^{[k+1]}=\bm{\widehat{H}}_{\rm ba}^{[k+1]} \mid \bm{\widehat{H}}^{[k]}=\bm{\widehat{H}}_{\rm ba}^{[k]} \big] \\
    =& \operatorname{Cov} \big[(\alpha \bm{H}_{\rm ba}^{[k]}+\sqrt{1-\alpha^2} \bm{\Omega}_{1}^{[k+1]}+\bm{\widehat{Z}}) \mid \bm{\widehat{H}}_{\rm ba}^{[k]} \big] \\
    =& \alpha^2 \cdot \operatorname{Cov} \big[\bm{H}_{\rm{ba}}^{[k]} \mid \widehat{\bm{H}}_{\rm{ba}}^{[k]} \big]\\
    & +(1-\alpha^2) \cdot \operatorname{Cov}\big[\bm{\Omega}_{1}^{[k+1]} \big]+\operatorname{Cov} \big[\bm{\widehat{Z}} \big] \\
    =& \alpha^2 \cdot [\bm{\Sigma}_{\bm{H}}-\bm{\Sigma}_{\bm{H}}\bm{\Sigma}_{\bm{\widehat{H}}}^{-1} \bm{\Sigma}_{\bm{H}}]+(1-\alpha^2) \cdot \bm{\Sigma}_{\bm{H}} + \sigma^2 \bm{I} \\
    \triangleq & \bm{\Sigma}_{\rm ba}.
\end{aligned}
\end{equation}
Notice that the conditional covariance does not depend on the index $[k+1]$ and, thus, is denoted simply by $\bm{\Sigma}_{\rm ba}$. Therefore, the probability density function of the null hypothesis, $\mathcal{H}_0$, can be written as
\begin{equation}
\begin{aligned}
    & f(\bm{\widehat{H}}^{[k+1]}=\bm{\widehat{H}}_{\rm ba}^{[k+1]}|\bm{\widehat{H}}^{[k]}=\bm{\widehat{H}}_{\rm ba}^{[k]}) \\
    & =\frac{1}{\pi^n \operatorname{det}(\bm{\Sigma}_{\rm ba})} e^{-(\bm{\widehat{H}}^{[k+1]}-\bm{\mu}_{\rm ba}^{[k+1]})^{H} \bm{\Sigma}_{\rm ba}^{-1} (\bm{\widehat{H}}^{[k+1]}-\bm{\mu}_{\rm ba}^{[k+1]})}.
\end{aligned}
\end{equation}

Similarly, under alternative hypothesis $\mathcal{H}_1$, the conditional distribution of $\bm{\widehat{H}}_{\rm ma}^{[k+1]}$ given $\bm{\widehat{H}}_{\rm ba}^{[k]}$ is
\begin{equation}
    \bm{\widehat{H}}_{\rm ma}^{[k+1]} \mid \bm{\widehat{H}}_{\rm ba}^{[k]} \sim \mathcal{CN}(\bm{\mu}_{\rm ma}^{[k+1]}, \bm{\Sigma}_{\rm ma}^{[k+1]}),
\end{equation}
where conditional mean
\begin{equation}
\label{eq:conditionalmean_ma}
\begin{aligned}
    \bm{\mu}_{\rm ma}^{[k+1]}& =\mathbb{E} \big[ \bm{\widehat{H}}^{[k+1]}=\bm{\widehat{H}}_{\rm ma}^{[k+1]} \mid \bm{\widehat{H}}^{[k]}=\bm{\widehat{H}}_{\rm ba}^{[k]}]\\
    & = \frac{\beta}{\sqrt{\Theta}} \cdot \bm{\Sigma}_{\bm{H}}\bm{\Sigma}_{\bm{\widehat{H}}}^{-1} \widehat{\bm{H}}_{\rm{ba}}^{[k]},
\end{aligned}
\end{equation}
and conditional covariance
\begin{equation}
\label{eq:covariance_ma}
\begin{aligned}
    \bm{\Sigma}_{\rm ma}^{[k+1]} =& \operatorname{Cov}[\bm{\widehat{H}}^{[k+1]}=\bm{\widehat{H}}_{\rm ma}^{[k+1]} \mid \bm{\widehat{H}}^{[k]}=\bm{\widehat{H}}_{\rm ba}^{[k]}]\\
    =& \frac{\beta^2}{\Theta} \cdot [\bm{\Sigma}_{\bm{H}}-\bm{\Sigma}_{\bm{H}} \bm{\Sigma}_{\bm{\widehat{H}}}^{-1} \bm{\Sigma}_{\bm{H}}]+\frac{1-\beta^2}{\Theta} \cdot \bm{\Sigma}_{\bm{H}}+\sigma^2 \bm{I} \\
    \triangleq & \bm{\Sigma}_{\rm ma}.
\end{aligned}
\end{equation}
Same as $\bm{\Sigma}_{\rm ba}$, the conditional covariance, $\bm{\Sigma}_{\rm ma}^{[k+1]}$, does not depend on the index $[k+1]$ and, thus, is simplified as $\bm{\Sigma}_{\rm ma}$. The probability density function of the alternative hypothesis, $\mathcal{H}_1$, can be written as
\begin{equation}
\begin{aligned}
    & f(\bm{\widehat{H}}^{[k+1]}=\bm{\widehat{H}}_{\rm ma}^{[k+1]}|\bm{\widehat{H}}^{[k]}=\bm{\widehat{H}}_{\rm ba}^{[k]}) \\
    & =\frac{1}{\pi^n \operatorname{det}(\bm{\Sigma}_{\rm ma})} e^{-(\bm{\widehat{H}}^{[k+1]}-\bm{\mu}_{\rm ma}^{[k+1]})^{H} \bm{\Sigma}_{\rm ma}^{-1} (\bm{\widehat{H}}^{[k+1]}-\bm{\mu}_{\rm ma}^{[k+1]})}.
\end{aligned}
\end{equation}

The NP detection is conducted based on the likelihood ratio, which is given as
\begin{equation}
\label{eq:likelihoodratio}
    \frac{f(\bm{\widehat{H}}^{[k+1]}=\bm{\widehat{H}}_{\rm ba}^{[k+1]}|\bm{\widehat{H}}^{[k]}=\bm{\widehat{H}}_{\rm ba}^{[k]})}{f(\bm{\widehat{H}}^{[k+1]}=\bm{\widehat{H}}_{\rm ma}^{[k+1]}|\bm{\widehat{H}}^{[k]}=\bm{\widehat{H}}_{\rm ba}^{[k]})} \underset{\mathcal{H}_1}{\overset{\mathcal{H}_0}{\gtrless}} \mathcal{T}^{\prime}.
\end{equation}
By applying the logarithm to both sides and introducing the definition $\mathcal{T} \triangleq \ln{(\mathcal{T}^{\prime} \cdot \operatorname{det}(\bm{\Sigma}_{\rm ba})/\operatorname{det}(\bm{\Sigma}_{\rm ma}))}$, \eqref{eq:likelihoodratio} can be derived into
\begin{equation}
\begin{aligned}
& -(\widehat{\bm{H}}^{[k+1]})^{H} \bm{\Sigma}_{\rm ba}^{-1} \widehat{\bm{H}}^{[k+1]}+(\widehat{\bm{H}}^{[k+1]})^{H} \bm{\Sigma}_{\rm ba}^{-1} \bm{\mu}_{\rm ba}^{[k+1]}\\
& +(\bm{\mu}_{\rm ba}^{[k+1]})^{H} \bm{\Sigma}_{\rm ba}^{-1} \widehat{\bm{H}}^{[k+1]}-(\bm{\mu}_{\rm ba}^{[k+1]})^{H} \bm{\Sigma}_{\rm ba}^{-1} \bm{\mu}_{\rm ba}^{[k+1]} \\
& +(\widehat{\bm{H}}^{[k+1]})^{H} \bm{\Sigma}_{\rm ma}^{-1} \widehat{\bm{H}}^{[k+1]}-(\widehat{\bm{H}}^{[k+1]})^{H} \bm{\Sigma}_{\rm ma}^{-1} \bm{\mu}_{\rm ma}^{[k+1]}\\
& -(\bm{\mu}_{\rm ma}^{[k+1]})^{H} \bm{\Sigma}_{\rm ma}^{-1} \widehat{\bm{H}}^{[k+1]}+(\bm{\mu}_{\rm ma}^{[k+1]})^{H} \bm{\Sigma}_{\rm ma}^{-1} \bm{\mu}_{\rm ma}^{[k+1]} \underset{\mathcal{H}_1}{\overset{\mathcal{H}_0}{\gtrless}} \mathcal{T},
\end{aligned}
\end{equation}
which can be further written as
\begin{equation}
\begin{aligned}
& -(\widehat{\bm{H}}^{[k+1]})^{H}(\bm{\Sigma}_{\rm ba}^{-1}-\bm{\Sigma}_{\rm ma}^{-1}) \widehat{\bm{H}}^{[k+1]}\\
& +2 \Re \{((\bm{\mu}_{\rm ba}^{[k+1]})^{H} \bm{\Sigma}_{\rm ba}^{-1}-(\bm{\mu}_{\rm ma}^{[k+1]})^{H} \bm{\Sigma}_{\rm ma}^{-1}) \widehat{\bm{H}}^{[k+1]}\} \\
& -(\bm{\mu}_{\rm ba}^{[k+1]})^{H} \bm{\Sigma}_{\rm ba}^{-1} \bm{\mu}_{\rm ba}^{[k+1]}+(\bm{\mu}_{\rm ma}^{[k+1]})^{H} \bm{\Sigma}_{\rm ma}^{-1} \bm{\mu}_{\rm ma}^{[k+1]}\underset{\mathcal{H}_1}{\overset{\mathcal{H}_0}{\gtrless}} \mathcal{T}.
\end{aligned}
\end{equation}
By substituting~\eqref{eq:conditionalmean_ba} and~\eqref{eq:conditionalmean_ma}, we can derive
\begin{equation}
\begin{aligned}
& (\bm{\widehat{H}}^{[k+1]})^{H} \bm{A} \bm{\widehat{H}}^{[k+1]}+\Re\{(\bm{\widehat{H}}_{\rm ba}^{[k]})^{H}\bm{B}\bm{\widehat{H}}^{[k+1]}\}\\
& +(\bm{\widehat{H}}_{\rm ba}^{[k]})^{H} \bm{C} \bm{\widehat{H}}_{\rm ba}^{[k]} \underset{\mathcal{H}_1}{\overset{\mathcal{H}_0}{\gtrless}}  \mathcal{T},
\end{aligned}
\end{equation}
where $\Re(\cdot)$ represents the real part and 
\begin{subequations}
\begin{align}
    & \bm{A}=-(\bm{\Sigma}_{\rm ba}^{-1}-\bm{\Sigma}_{\rm ma}^{-1}),\\
    & \bm{B}=2(\bm{\Sigma}_{\bm{H}} \bm{\Sigma}_{\bm{\widehat{H}}}^{-1})^{H}(\alpha \bm{\Sigma}_{\rm ba}^{-1}-\frac{\beta}{\sqrt{\Theta}}\bm{\Sigma}_{\rm ma}^{-1}),\\
    & \bm{C}=-\bm{\Sigma}_{\bm{\widehat{H}}}^{-1} \bm{\Sigma}_{\bm{H}} (\alpha^2 \bm{\Sigma}_{\rm ba}^{-1}-\frac{\beta^2}{\Theta} \bm{\Sigma}_{\rm ma}^{-1}) \bm{\Sigma}_{\bm{H}} \bm{\Sigma}_{\bm{\widehat{H}}}^{-1},
\end{align}
\end{subequations}
are related to the channel statistics.
\end{proof}
}

\section{The proof of the equivalence between~\eqref{eq:hypothesiscomplex} and~\eqref{eq:hypothesisrealimag}}
\label{proof:model}
\begin{proof}
Note that both the first term $(\bm{\widehat{H}}^{[k+1]})^{H} \bm{A} \bm{\widehat{H}}^{[k+1]}$ and the third term $(\bm{\widehat{H}}_{\rm ba}^{[k]})^{H} \bm{C} \bm{\widehat{H}}_{\rm ba}^{[k]}$ in the left-hand side of~\eqref{eq:hypothesiscomplex} are Hermitian quadratic forms, where $\bm{A}$ and $\bm{C}$ are Hermitian matrices. Therefore, these terms are real-valued and~\eqref{eq:hypothesiscomplex} can be rewritten as 
\begin{equation}
\label{eq:hypothesisequivalentcomplex}
\begin{aligned}
& \Re\{(\bm{\widehat{H}}^{[k+1]})^{H} \bm{A} \bm{\widehat{H}}^{[k+1]}\}+\Re\{(\bm{\widehat{H}}_{\rm ba}^{[k]})^{H}\bm{B}\bm{\widehat{H}}^{[k+1]}\}\\
& +\Re\{(\bm{\widehat{H}}_{\rm ba}^{[k]})^{H} \bm{C} \bm{\widehat{H}}_{\rm ba}^{[k]}\} \underset{\mathcal{H}_1}{\overset{\mathcal{H}_0}{\gtrless}}  \mathcal{T},
\end{aligned}
\end{equation}
where the three terms on the left-hand side have the same structure. Without loss of generality, $\forall\ \bm{h}_{1}, \bm{h}_{2}\in \mathbb{C}^{M^{\prime}\times 1}$ and $\bm{K}\in \mathbb{C}^{M^{\prime}\times M^{\prime}}$, we have
\begin{equation}
\label{eq:complextorealimag}
\begin{aligned}
&\Re\{\bm{h}_{1}^{H} \bm{K} \bm{h}_{2}\} \\
= &\Re\{(\Re\{ \bm{h}_{1} \}+j\Im\{ \bm{h}_{1} \})^{H} (\Re\{ \bm{K} \}+j\Im\{ \bm{K} \}) \\
& (\Re\{ \bm{h}_{2} \}+j\Im\{ \bm{h}_{2} \}) \} \\
= &\Re\{ \bm{h}_{1} \}^{T}(\Re\{ \bm{K} \}\Re\{ \bm{h}_{2} \}-\Im\{ \bm{K} \}\Im\{ \bm{h}_{2} \}) + \\
&\Im\{ \bm{h}_{1} \}^{T} (\Im\{ \bm{K} \}\Re\{ \bm{h}_{2} \}+\Re\{ \bm{K} \}\Im\{ \bm{h}_{2} \}) \\
= & \begin{bmatrix}
\Re\{ \bm{h}_{1} \}\\
\Im\{ \bm{h}_{1} \}
\end{bmatrix}^{T}
\begin{bmatrix}
\Re\{ \bm{K} \}&-\Im\{ \bm{K} \}\\
\Im\{ \bm{K} \}&\Re\{ \bm{K} \}\\
\end{bmatrix}
\begin{bmatrix}
\Re\{ \bm{h}_{2} \}\\
\Im\{ \bm{h}_{2} \}
\end{bmatrix}.
\end{aligned}
\end{equation}
By applying the result of~\eqref{eq:complextorealimag} to~\eqref{eq:hypothesisequivalentcomplex}, we obtain~\eqref{eq:hypothesisrealimag}.
\end{proof}

\section{The proof of the rank of matrices $\bm{A}^{\prime}$, $\bm{B}^{\prime}$ and $\bm{C}^{\prime}$}
\label{proof:rank}
\begin{proof} 
First, we discuss the rank of matrices $\bm{A}$, $\bm{B}$ and $\bm{C}$. From~\eqref{eq:covariance_ba} and~\eqref{eq:covariance_ma}, it can be derived that both $\bm{\Sigma}_{\rm ba}$ and $\bm{\Sigma}_{\rm ma}$ are linear combination of $\bm{\Sigma}_{\bm{H} \mid \widehat{\bm{H}}}$, $\bm{\Sigma}_{\bm{H}}$ and $\sigma^2 \bm{I}$, with
\begin{equation}
    \bm{\Sigma}_{\rm ba}\succ\bm{0},\quad \bm{\Sigma}_{\rm ma}\succ\bm{0},
\end{equation}
which indicates that $\text{rank}(\bm{\Sigma}_{\rm ba})=\text{rank}(\bm{\Sigma}_{\rm ma})=M^{\prime}$ and $\bm{\Sigma}_{\rm ba}$ and $\bm{\Sigma}_{\rm ma}$ share the same basis vectors. Further, we have $\text{rank}(\bm{\Sigma}_{\rm ba}^{-1})=\text{rank}(\bm{\Sigma}_{\rm ma}^{-1})=M^{\prime}$, with $\bm{\Sigma}_{\rm ba}^{-1}$ and $\bm{\Sigma}_{\rm ma}^{-1}$ sharing the same basis vectors. Considering that $\bm{A}=-(\bm{\Sigma}_{\rm ba}^{-1}-\bm{\Sigma}_{\rm ma}^{-1})$ and $\alpha^2\neq \frac{\beta^2}{\Theta}$ in practice, $\bm{A}$ is also a full rank matrix, which means $\text{rank}(\bm{A})=M^{\prime}$.

For $\bm{B}$ and $\bm{C}$, we define a steering vector
\begin{equation}
\bm{v}_l=
\begin{bmatrix}
    1&e^{-j2\pi l/M}&\cdots&e^{-j2\pi(M-1)l/M}
\end{bmatrix}^{T},
\end{equation}
which is a column vector of the DFT matrix. Then the covariance matrix $\bm{\Sigma}_{\bm{H}}$ can be rewritten as
\begin{equation}
\bm{\Sigma}_{\bm{H}}=\sum_{l=0}^{L-1} \sigma_{\rm{ba}}^2(l) \bm{v}_{l} \bm{v}_{l}^{H},
\end{equation}
where each term $\sigma_{\rm{ba}}^2(l) \bm{v}_{l} \bm{v}_{l}^{H}$ is a rank-one matrix with an eigen vector $\bm{v}_{l}$. Since DFT matrix is an unitary matrix, i.e., 
\begin{equation}
    \bm{v}_{i}^{H}\bm{v}_{j}=0, \forall\ i\neq j,
\end{equation}
it indicates that $\bm{v}_{l} \bm{v}_{l}^{H}, l=0, 1, \cdots, L-1$ are linear independent. Therefore, $\text{rank}(\bm{\Sigma}_{\bm{H}})=L$. Based on the properties of matrix rank, we have $\text{rank}(\bm{B})\leq L$, $\text{rank}(\bm{C})\leq L$.

Next, we prove $\text{rank}(\bm{A}^{\prime})=2\cdot \text{rank}(\bm{A})$. Let $\text{rank}(\bm{A})=r$. We denote the $j$-th column of $\bm{A}$ as
\begin{equation}
\bm{a}_{j}=\bm{x}_{j}+i\bm{y}_{j},
\end{equation}
where $\bm{x}_{j}$ and $\bm{y}_{j}$ are real and imaginary parts of $\bm{a}_{j}$, respectively. The corresponding columns in $\bm{A}^{\prime}$ are represented as
\begin{equation}
    \begin{bmatrix}
    \bm{x}_{j}\\
    \bm{y}_{j}
\end{bmatrix},
\begin{bmatrix}
    -\bm{y}_{j}\\
    \bm{x}_{j}
\end{bmatrix}.
\end{equation}
We assume that $\bm{a}_{1}, \bm{a}_{2}, \cdots, \bm{a}_{r}$ is a set of linearly independent column vectors of $\bm{A}$, then the corresponding $2r$ real vectors can be denoted as 
\begin{equation}
\begin{bmatrix}
\bm{x}_1 \\ \bm{y}_1
\end{bmatrix}, 
\begin{bmatrix}
-\bm{y}_1 \\ \bm{x}_1
\end{bmatrix}, \cdots,
\begin{bmatrix}
\bm{x}_r \\ \bm{y}_r
\end{bmatrix}, 
\begin{bmatrix}
-\bm{y}_r \\ \bm{x}_r
\end{bmatrix}.
\end{equation}
If there exists a linear combination of these real column vectors equal to zero
\begin{equation}
\sum_{j=1}^r \alpha_j
\begin{bmatrix}
\bm{x}_j \\
\bm{y}_j
\end{bmatrix}+\sum_{j=1}^r \beta_j
\begin{bmatrix}
-\bm{y}_j \\
\bm{x}_j 
\end{bmatrix}=\bm{0},
\end{equation}
then by combining the real and imaginary parts, we get
\begin{equation}
\sum_{j=1}^r\left(\alpha_j+i \beta_j\right) \bm{a}_j=\bm{0}.
\end{equation}
Since the complex column vectors $\bm{a}_j$ are linearly independent in $\mathbb{C}^{M^{\prime}}$, this implies that 
\begin{equation}
    \alpha_j+i \beta_j=0,
\end{equation}
which gives 
\begin{equation}
    \alpha_j=0, \ \beta_j=0, \ \forall j.
\end{equation}
Thus, the $2r$ real vectors corresponding to $\bm{a}_{1}, \bm{a}_{2}, \cdots, \bm{a}_{r}$ in $\bm{A}^{\prime}$ are linearly independent in $\mathbb{R}^{2M^{\prime}}$. Therefore, 
\begin{equation}
    \text{rank}(\bm{A}^{\prime})\geq 2\cdot \text{rank}(\bm{A}).
\end{equation}
Suppose $\text{rank}(\bm{A}^{\prime})>2r$, then there exist at least $2r+1$ linearly independent real column vectors in $\bm{A}^{\prime}$. Since each complex column of $\bm{A}$ generates two real columns in $\bm{A}^{\prime}$, the linear independence of $2r+1$ real columns in $\bm{A}^{\prime}$ would imply the existence of $r+1$ linearly independent complex columns in $\bm{A}$, which contradicts the assumption that $\text{rank}(\bm{A})=r$. Therefore, 
\begin{equation}
    \text{rank}(\bm{A}^{\prime})=2\cdot \text{rank}(\bm{A})=2M^{\prime}.
\end{equation}
Similarly,
\begin{subequations}
\begin{align}
    & \text{rank}(\bm{B}^{\prime})=2\cdot \text{rank}(\bm{B})\leq 2L,\\
    & \text{rank}(\bm{C}^{\prime})=2\cdot \text{rank}(\bm{C})\leq 2L.
\end{align}
\end{subequations}
\end{proof}

\bibliographystyle{IEEEtran}
\bibliography{IEEEabrv,cites}

\begin{IEEEbiographynophoto}{Yijia Guo} (Member, IEEE)
received the M.Eng. degree in information and communication engineering from Sun Yat-sen University, Guangzhou, China, in 2021. She is currently pursuing the Ph.D. degree with the School of Computer Science and Informatics, University of Liverpool, Liverpool, UK. Her current research interests include the Internet of Things, deep learning and physical layer authentication.
\end{IEEEbiographynophoto}

\begin{IEEEbiographynophoto}{Junqing Zhang} (Senior Member, IEEE) received B.Eng and M.Eng degrees in Electrical Engineering from Tianjin University, China in 2009 and 2012, respectively, and a Ph.D. degree in Electronics and Electrical Engineering from Queen's University Belfast, UK in 2016. From Feb. 2016 to Jan. 2018, he was a Postdoctoral Research Fellow at Queen's University Belfast. From Feb. 2018 to Oct. 2022, he was a Tenure Track Fellow and then a Lecturer (Assistant Professor) at the University of Liverpool, UK. Since Oct. 2022, he has been a Senior Lecturer (Associate Professor) at the University of Liverpool. His research interests include the Internet of Things, wireless security, physical layer security, key generation, radio frequency fingerprint identification, and wireless sensing. 
Dr. Zhang is a co-recipient of the IEEE WCNC 2025 Best Workshop Paper Award. He is a Senior Area Editor of IEEE Transactions on Information Forensics and Security and an Associate Editor of IEEE Transactions on Mobile Computing.
\end{IEEEbiographynophoto}

\begin{IEEEbiographynophoto}{Y.-W. Peter Hong} (S'01 - M'05 - SM'13) 
received the B.S. degree in electrical engineering from National Taiwan University, Taipei, Taiwan, in 1999, and the Ph.D. degree in electrical engineering from Cornell University, Ithaca, NY, USA, in 2005. He is currently a Distinguished Professor with the Institute of Communications Engineering and the Department of Electrical Engineering at National Tsing Hua University (NTHU), Hsinchu, Taiwan. His research interests include AI/ML for wireless communications, UAV and satellite communications, distributed signal processing for the IoT and sensor networks, and physical layer security.

Dr. Hong received the IEEE ComSoc Asia-Pacific Outstanding Young Researcher Award in 2010, the Y. Z. Hsu Scientific Paper Award in 2011, the National Science Council Wu Ta-You Memorial Award in 2011, the Chinese Institute of Electrical Engineering (CIEE) Outstanding Young Electrical Engineer Award in 2012, and the National Science and Technology Council (NSTC) Outstanding Research Award in 2018 and 2022. He was the Chair of the IEEE ComSoc Taipei Chapter from 2017 to 2018. He was the Co-Chair of the Technical Affairs Committee, the Information Services Committee, and the Chapter Coordination Committee of the IEEE ComSoc Asia-Pacific Board, from 2014 to 2015, from 2016 to 2019, and from 2020 to 2021, respectively. He is now the Vice Director of the IEEE ComSoc Asia-Pacific Board from 2022 to 2027. In the past, he also served as Senior Area Editor and Associate Editor, respectively, for IEEE Transactions on Signal Processing, Associate Editor for IEEE Transactions on Information Forensics and Security and an Editor for IEEE Transactions on Communications. He was also a Distinguished Lecturer of the IEEE Communications Society from 2022 to 2023.
\end{IEEEbiographynophoto}

\begin{IEEEbiographynophoto}{Stefano Tomasin} (Senior Member, IEEE) received the Ph.D. degree from the University of Padova, Italy, in 2003. During his studies, he did internships with IBM Research (Switzerland) and Philips Research (Netherlands). He joined the University of Padova, where he has been Assistant Professor (2005-2015), Associate Professor (2016-2022), and Full Professor (since 2022). He was visiting faculty at Qualcomm, San Diego (CA) in 2004, the Polytechnic University in Brooklyn (NY) in 2007, and the Mathematical and Algorithmic Sciences Laboratory of Huawei in Paris (France) in 2015. His current research interests include physical layer security, security of global navigation satellite systems, signal processing for wireless communications, synchronization, and scheduling of communication resources. He has been a senior member of IEEE since 2011 (member since 1999) and a member of EURASIP since 2011. He is or has been an Editor of the IEEE Transactions on Vehicular Technologies (2011-2016), of the IEEE Transactions on Signal Processing (2017-2020), of the EURASIP Journal of Wireless Communications and Networking (since 2011), and of the IEEE Transactions on Information Forensics and Security (since 2020). He also serves as a Deputy Editor-in-Chief of the IEEE Transactions on Information Forensics and Security since January 2023.
\end{IEEEbiographynophoto}

\end{document}